\let\oldvec\vec%
\documentclass[envcountsame]{llncs}
\let\vec\oldvec%
    \usepackage{subfigure}
    \usepackage[dvipsnames,usenames]{xcolor}
    \usepackage{tikz}
      \usetikzlibrary{arrows}
      \usetikzlibrary{decorations.pathmorphing}
      \usetikzlibrary{decorations.shapes}
      \usetikzlibrary{backgrounds}
      \usetikzlibrary{positioning}
      \usetikzlibrary{fit}

    \usepackage{pgffor}
    \usepackage{ntabbing}
    \usepackage{ifthen}
    \usepackage{scrtime}    
    \usepackage{amsmath} 
    \usepackage{amsfonts} 
    \usepackage{amssymb}   
    \usepackage{mathtools} 

    \usepackage{graphicx}
    \usepackage{setspace}
    \usepackage{paralist}
    \usepackage{algorithmicx}
    \usepackage{algpseudocode}
    \usepackage{algorithm}
    \usepackage{etoolbox,fp}

\newcommand{\struc}[1]{\mathsf #1}

\newcommand{\ie}{\mbox{i.\,e.}} 

\newcommand{\CSP}{{\sf\upshape CSP}}

\newcommand{\kcons}{$k$-{\sf\upshape Cons}}
\newcommand{\kCons}{$k$-{\sf\upshape Cons}}
\newcommand{\twoCons}{$2$-{\sf\upshape Cons}}

\newcommand{\width}{\operatorname{width}}
\newcommand{\depth}{\operatorname{depth}}
\newcommand{\Prop}{\operatorname{Prop}}

\newcommand{\INC}{\text{\upshape INC}}
\newcommand{\WIN}{\text{\upshape WIN}}

\newcommand\restr[2]{{
  \left.\kern-\nulldelimiterspace 
  #1 
  \vphantom{\big|} 
  \right|_{#2} 
  }}

\newcommand{\kmo}{\mathsf k}

\colorlet{acolor}{Apricot!30}
\colorlet{bcolor}{blue}
\colorlet{ccolor}{cyan}
\colorlet{dcolor}{Dandelion}

\tikzset{
a/.style={->},
aa/.style={<-},
aaa/.style={->},
aarrow/.style={draw=blue, ultra thick,<-},
barrow/.style={draw=Red!70, ultra thick,<-},
  avertex/.style={circle,draw,inner sep=2pt,fill=acolor},
  aavertex/.style={circle,draw,inner sep=2.5pt,fill=Apricot},
bvertex/.style={circle,draw,inner sep=2pt,fill=bcolor},
rvertex/.style={regular polygon,regular polygon sides=6,draw,inner sep=2pt,fill=Red},
cvertex/.style={circle,draw,inner sep=2pt,fill=ccolor},
dvertex/.style={regular polygon,regular polygon sides=6,draw,inner sep=2.1pt,fill=dcolor},
  uvertex/.style={circle,draw=Black,inner sep=2pt,fill=white},
  schvertex/.style={regular polygon,regular polygon sides=6,draw,inner sep=2pt,fill=Red},
  vertex/.style={circle,draw=Gray,inner sep=2pt,fill=Gray}
}

\newcommand{\defi}{:=}
\newcommand{\deffi}{:=}
\renewcommand{\defi}{\deffi}

    \newcommand{\Kommentar}[1]{}
    
    \renewcommand{\epsilon}{\varepsilon}

\allowdisplaybreaks

    \newcommand{\posq}{\operatorname{\mathfrak{q}}}
    \newcommand{\posa}{\operatorname{\mathfrak{a}}}
    \newcommand{\posb}{\operatorname{\mathfrak{b}}}
    
    \newcommand{\cl}{\wp}
    \newcommand{\crit}{\operatorname{crit}}

    \newcommand{\whitenode}{\tikz{\node[circle,draw=black,fill=white,inner  sep=0pt,minimum  size=1mm] at (0,0){};} }
    \newcommand{\blacknode}{\tikz{\node[circle,draw=black,fill=black,inner  sep=0pt,minimum  size=1mm] at (0,0){};} }
    \newcommand{\xnode}{\mathsf x}
    \newcommand{\ynode}{\mathsf y}
    \renewcommand{\circ}{\uplus}

    \newcommand{\LOGSPACE}{{\upshape LOGSPACE}}
    \newcommand{\EXPTIME}{{\upshape EXPTIME}}

		\newcommand{\PTIME}{{\upshape PTIME}}
    \newcommand{\NCclass}{{\upshape NC}}
    \newcommand{\XP}{{\upshape XP}}

		\newcommand{\problembox}[3]{\smallskip\par
			\begin{center}
				\fbox{\parbox{.9\columnwidth}{
					#1\par \vspace{2mm}
					\begin{tabular}{rl}
							\textit{Input}: & #2\\
							\textit{Question}: & #3 
					\end{tabular}
				}}
			\end{center}\smallskip
		}

    \newcommand{\dom}{\operatorname{Dom}}

    \newenvironment{innerproof}{\begin{proof}}{\end{proof}}
    
\newenvironment{IEEEproof}{\proof}{\qed}

\newcommand{\insubmission}[1]{#1}
\renewcommand{\insubmission}[1]{}

\makeatletter
\newlength\tdima
\newcommand\tabright[1]{%
      \setlength\tdima{\linewidth}%
      \addtolength\tdima{\@totalleftmargin}%
      \addtolength\tdima{-\dimen\@curtab}%
      \makebox[\tdima][r]{#1}}
\makeatother

\newcommand{\cref}{\simeq}

\def\comm#1{}
\def\incomm#1{}

\newcommand{\spaceconstraints}[1]{}

\newcommand{\confORarxiv}[2]{#2}

		\newcommand{\tikzinput}[1]{\includegraphics{TikZ/keinBild.png}}
		\renewcommand{\tikzinput}[1]{\input{#1}}
		
		\newcounter{todos}

\title{The Propagation Depth of Local Consistency}

\author{Christoph Berkholz}

\institute{RWTH Aachen University, Aachen, Germany}

\begin{document}

\frontmatter          %

\maketitle

\begin{abstract}
We establish optimal bounds on the number of nested propagation steps in $k$-consistency tests. 
It is known that local consistency algorithms such as arc-, path- and $k$-consistency are not efficiently parallelizable. 
Their inherent sequential nature is caused by long chains of nested propagation steps, which cannot be executed in parallel. 
This motivates the question ``What is the minimum number of nested propagation steps that have to be performed by $k$-consistency algorithms on (binary) constraint networks with $n$ variables and domain size $d$?'' 

It was known before that 2-consistency requires $\Omega(nd)$ and 3-consistency requires $\Omega(n^2)$ sequential propagation steps. 
We answer the question exhaustively for every $k\geq 2$: there are binary constraint networks where any $k$-consistency procedure has to perform $\Omega(n^{k-1}d^{k-1})$ nested propagation steps before local inconsistencies were detected. 
This bound is tight, because the overall number of propagation steps performed by $k$-consistency is at most $n^{k-1}d^{k-1}$. 
\end{abstract}

\section{Introduction}
\label{sec:CSP-Intro}
A constraint network $(X,D,C)$ consists of a set $X$ of $n$ variables over a domain $D$ of size $d$ and a set of constraints $C$ that restrict possible assignments of the variables. 
The \emph{constraint satisfaction problem} (\CSP{}) is to find an assignment of the variables with values from $D$ such that all constraints are satisfied. 
The constraint satisfaction problem can be solved in exponential time by exhaustive search over all possible assignments.
\emph{Constraint propagation} is a technique to speed up the exhaustive search by restricting the search space in advance. This is done by iteratively propagating new constraints that follow from previous ones. 
Most notably, in local consistency algorithms the overall goal is to propagate new constraints to achieve some kind of consistency on small parts of the constraint network. Additionally, if local inconsistencies were detected, it follows that the constraint network is also globally inconsistent and hence unsatisfiable.

The $k$-consistency test \cite{Freuder.1978} is a well-known local consistency technique, which enforces that every satisfying $(k-1)$-partial assignment can be extended to a satisfying $k$-partial assignment. 
At the beginning, all partial assignments that violate a constraint were marked as inconsistent. Then the following inference rule is applied iteratively:

\begin{quote}\label{quote:Consistency-rule}
If $h$ is a consistent $\ell$-partial assignment ($\ell<k$) for which there exists a variable $x\in X$ such that $h\cup\{x\mapsto a\}$ is inconsistent for all $a\in D$, then mark $h$ and all its extensions as inconsistent. 
\end{quote}

After at most $n^{k-1}d^{k-1}$ propagation steps this procedure stops. If the empty assignment becomes inconsistent, we say that (strong) $k$-consistency \emph{cannot be established}. In this case we know that the constraint network is globally inconsistent. 
Otherwise, if $k$-consistency \emph{can be established}, we can use the propagated constraints to restrict the search space for a subsequent exhaustive search. 
There are several different $k$-consistency algorithms in the literature, especially for $k=2$ (arc consistency) and $k=3$ (path consistency), which all follow this propagation scheme.
The main difference between these algorithms are the underlying data structure and the order in which they apply the propagation rule.
It seems plausible to apply the propagation rule in parallel in order to detect local inconsistencies in different parts of the constraint network at the same time. 
Indeed, this intuition has been used to design parallel arc and path consistency algorithms \cite{Samal.1987,Susswein.1991}. On the other hand, the $k$-consistency test is known to be \PTIME{}-complete \cite{Kasif.1990,Kolaitis.2003} and hence not efficiently parallelizable (unless \NCclass{}=\PTIME). 
The main bottleneck for parallel approaches are the sequential dependencies in the propagation rule: some assignments will be marked as inconsistent after some other assignments became inconsistent.

For 2-consistency the occurrence of long chains of sequential dependencies has been observed very early \cite{Dechter.1985} and was recently studied in depth in \cite{BerVer13}. There are simple constraint networks for which 2-consistency requires $\Omega(nd)$ nested propagation steps. 
Ladkin and Maddux \cite{Ladkin.1994} used algebraic techniques to show that 3-consistency requires $\Omega(n^{2})$ nested propagation steps on binary constraint networks with constant domain.
We extend these previous results and obtain a complete picture of the propagation depth of $k$-consistency.
Our main result (Theorem~\ref{thm:kConsDepthLowerBound}) states that for every constant $k\geq 2$ and given integers $n$, $d$ there is a constraint network with $n$ variables and domain size $d$ such that every $k$-consistency algorithm has to perform $\Omega(n^{k-1}d^{k-1})$ nested propagation steps.
This lower bound is optimal as it is matched by the trivial upper bound $n^{k-1}d^{k-1}$ on the overall number of propagation steps. It follows that every parallel propagation algorithm for $k$-consistency has a worst case time complexity of $\Omega(n^{k-1}d^{k-1})$. Since the best-known running time of a sequential algorithm for $k$-consistency is $O(n^kd^k)$ \cite{Cooper.1989} it follows that no significant improvement over the sequential algorithm is possible. 

\section{Preliminaries}
As first pointed out by Feder and Vardi \cite{Feder.1998} the \CSP{} is equivalent to the structure homomorphism problem where two finite relational structures $\struc{A}$ and $\struc{B}$ are given as input. 
The universe $V(\struc{A})$ of structure $\struc{A}$ corresponds to the set of variables $X$ and the universe $V(\struc{B})$ of structure $\struc{B}$ corresponds to the domain $D$. 
The constraints are encoded into relations such that every homomorphism from $\struc{A}$ to $\struc{B}$ corresponds to a solution of the \CSP{}. 
For the rest of this paper we mainly stick to this definition as it is more convenient to us. 
In fact, our main result benefits to a large extend from the fruitful connection between these two viewpoints. 

In the introduction we have presented $k$-consistency as a propagation procedure on constraint networks. 
Below we restate the definition in terms of a formal inference system (which is inspired by the proof system in \cite{Atserias.2004} and is a generalization of \cite{BerVer13}). 
This view allows us to gain insight into the structure of the propagation process and to formally state our main theorem afterwards. 
At the end of this section we provide a third characterization of $k$-consistency in terms of the existential pebble game, which is the tool of our choice in the proof of the main theorem.

\subsection{CSP-refutations}

Given two $\sigma$-structures $\struc{A}$ and $\struc{B}$, every line of our derivation system is a partial mapping from $V(\struc{A})$ to $V(\struc{B})$. The axioms are all partial mappings $p\colon{}V(\struc{A})\to V(\struc{B})$ that are not partial homomorphisms. We have the following derivation rule to derive a new inconsistent assignment $p$. For all partial mappings $p'_i\subseteq p$, $x\in V(\struc{A})$ and $V(\struc{B})=\{a_1,\ldots,a_n\}$:
\begin{align}
	&\frac{p'_1\cup\{x\mapsto a_1\}\quad \cdots \quad p'_n\cup\{x\mapsto a_n\}}{p}
	\label{eq:CSPrule1}
\end{align}
A \emph{CSP-derivation} of $p$ is a sequence $(p_1,\ldots,p_\ell=p)$ such that every $p_i$ is either an axiom or derived from lines $p_j$, $j<i$, via the derivation rule \eqref{eq:CSPrule1}. A \emph{CSP-refutation} is a CSP-derivation of $\emptyset$. Every derivation of $p$ can naturally be seen as a directed acyclic graph (dag) where the nodes are labeled with lines from the derivation, one node of in-degree 0 is labeled with $p$ and all nodes of out-degree 0 are labeled with axioms. If $p_i$ is derived from $p_{j_1},\ldots,p_{j_n}$ using \eqref{eq:CSPrule1}, then there is an arc from $p_i$ to each $p_{j_1},\ldots,p_{j_n}$. 

Given a CSP-derivation $P$, we let $\Prop(P)$ be the set of propagated mappings $p\in P$, \ie{} all lines in the derivation that are not axioms. We define the \emph{width} of a derivation $P$ to be $\width(P)=\max_{p\in \Prop(P)}|p|$.\footnote{Note that this implies $|p|\leq \width(P)+1$ for all axioms $p$ used in the derivation $P$. However, the size of the axioms can always be bounded by the maximum arity of the relations in $\struc{A}$ and $\struc{B}$.} 
Furthermore, $\depth(P)$ denotes the \emph{depth} of $P$ which is the number of edges on the longest path in the dag associated with $P$. This measure characterizes the maximum number of nested propagation steps in $P$. 
Since CSP-derivations model the propagation process mentioned in the introduction, there is a CSP-refutation of width $k-1$ if and only if $k$-consistency cannot be established.
Furthermore, every propagation algorithm produces some CSP-derivation $P$. 
The total number of propagation steps performed by this algorithm is $|\Prop(P)|$ and the maximum number of nested propagation steps is $\depth(P)$. 
Let $\struc{A}$ and $\struc{B}$ be two relational structures such that $k$-consistency cannot be established. We define the \emph{propagation depth} $\depth^k(\struc{A}, \struc{B}) \defi \min_P \depth(P)$
where the minimum is taken over all CSP-refutations $P$ of width at most $k-1$.  
Hence, the $\depth^k(\struc{A}, \struc{B}) \leq |V(\struc{A})|^{k-1}|V(\struc{B})|^{k-1}$
 is the number of sequential propagation steps that have to be performed by any sequential or parallel propagation algorithm for $k$-consistency. 

\subsection{Results and Related Work}
\label{sec:CSP-Intro-Results}

Our main theorem is a tight lower bound on the propagation depth. 
\begin{theorem}\label{thm:kConsDepthLowerBound}
	For every integer $k\geq 2$ there exists a constant $\epsilon>0$ and two positive integers $n_0$, $m_0$ such that for every $n\geq n_0$ and $m\geq m_0$ there exist two binary structures $\struc{A}_n$ and $\struc{B}_m$ with $|V(\struc{A}_n)|=n$ and $|V(\struc{B}_m)|=m$ such that $\depth^{k}(\struc{A}_n,\struc{B}_m)\geq \epsilon n^{k-1}m^{k-1}$.
\end{theorem}

We are aware of two particular cases that have been discovered earlier. 
First, for the case $k=2$ (arc consistency) the theorem can be shown by rather simple examples that occurred very early in the AI-community. The structure of this exceptional case is discussed in deep in a joint work of Oleg Verbitsky and the author of this paper \cite{BerVer13}. 
Second, for $k=3$ Ladkin and Maddux \cite{Ladkin.1994} showed that there is a fixed finite binary structure $\struc{B}$ and an infinite sequence of binary structures $\struc{A}_i$ such that  $\depth^3(\struc{A_i},\struc{B})=\Omega(|V(\struc{A}_i)|^2)$. They used this result to argue that every parallel propagation algorithm for path consistency needs at least a quadratic number of steps. This is tight only for fixed structures $\struc{B}$, Theorem \ref{thm:kConsDepthLowerBound} extends their result to the case when $\struc{B}$ is also given as input.

Other related results investigate the decision complexity of the $k$-consistency test. 
To address this more general question one analyzes the computational complexity of the following decision problem.
\problembox{\kcons}{Two binary relational structures $\struc{A}$ and $\struc{B}$.}{Can $k$-consistency be established for $\struc{A}$ and $\struc{B}$?}
Kasif \cite{Kasif.1990} showed that \twoCons{} is complete for \PTIME{} under \LOGSPACE{} reductions. 
Kolaitis and Panttaja \cite{Kolaitis.2003} extended this result to every fixed $k\geq 2$.
Moreover, they established that the problem is complete for \EXPTIME{} if $k$ is part of the input. 
In \cite{Ber13} the author showed that \kcons{} cannot be decided in $O(n^{\frac{k-3}{12}})$ on deterministic multi-tape Turing machines, where $n$ is the overall input size. Hence, any algorithm solving \kcons{} (regardless of whether it performs constraint propagation or not) cannot be much faster than the standard propagation approach. 
It also follows from this result that,
parameterized by the number of pebbles $k$, \kCons{} is is complete for the parameterized complexity class \XP. It is also worth noting that Gaspers and Szeider \cite{Gaspers.2011} investigated the parameterized complexity of other parameterized problems related to $k$-consistency.

\subsection{The Existential Pebble Game}
\label{CSP-Intro-ExPebbleGame}

In this paragraph we introduce a third view on the $k$-consistency heuristic in terms of a combinatorial pebble game.
The \emph{existential $k$-pebble game} \cite{Kolaitis.1995} is played by two players \emph{Spoiler} and \emph{Duplicator} on two relational structures $\struc{A}$ and $\struc{B}$. There are $k$ pairs of pebbles $(p_1,q_1),\ldots,(p_k,q_k)$ and during the game Spoiler moves the pebbles $p_1,\ldots,p_k$ to elements of $V(\struc{A})$ and Duplicator moves the pebbles $q_1,\ldots,q_k$ to elements of $V(\struc{B})$.
At the beginning of the game, Spoiler places pebbles $p_1,\ldots,p_{k}$ on elements of $V(\struc{A})$ and Duplicator answers by putting pebbles $q_1,\ldots,q_{k}$ on elements of $V(\struc{B})$. In each further round Spoiler picks up a pebble pair $(p_i,q_i)$ and places $p_i$ on some element in $V(\struc{A})$. Duplicator answers by moving the corresponding pebble $q_i$ to one element in $V(\struc{B})$. Spoiler wins the game if he can reach a position where the mapping defined by $p_i\mapsto q_i$ is not a partial homomorphism from $\struc{A}$ to $\struc{B}$. 

The connection between the existential $k$-pebble game and the $k$-consistency heuristic was made by Kolaitis and Vardi \cite{Kolaitis.2000a}. They showed that one can establish $k$-consistency by computing a winning strategy for Duplicator.
Going a different way, the next lemma states that there is also a tight correspondence between Spoiler's strategy and CSP-refutations. 
\confORarxiv{The proof is a straightforward induction over the depth and included in the full version of the paper.}{The proof is a straightforward induction over the depth and given in Appendix~\ref{sec:Alemma}.}
\begin{lemma}\label{lem:CSPrefutaionEXpebblegame}
	Let $\struc{A}$ and $\struc{B}$ be two relational structures. There is a CSP-refutation for $\struc{A}$ and $\struc{B}$ of width $k-1$ and depth $d$ if and only if Spoiler has a strategy to win the existential $k$-pebble game on $\struc{A}$ and $\struc{B}$ within $d$ rounds.
\end{lemma}
Using this lemma it suffices to prove lower bounds on the number of rounds in the existential pebble game in order to prove Theorem~\ref{thm:kConsDepthLowerBound}.  
To argue about strategies in the existential pebble game we use the framework developed in \cite{Ber13}.
We start with a formal definition of strategies for Duplicator. 
\begin{definition} \label{def:criticalStrategy}
 A \emph{critical strategy} for Duplicator in the existential $k$-pebble game on structures $\struc{A}$ and $\struc{B}$ is a nonempty family $\mathcal H$ of partial homomorphisms from $\struc{A}$ to $\struc{B}$ together with a set $\crit(\mathcal H)\subseteq \mathcal H$ of \textit{critical positions} satisfying the following properties:
\begin{enumerate} 
 \item All critical positions are $(k-1)$-partial homomorphisms.
 \item If $h\in \mathcal H$ and $g\subset h$, then $g\in\mathcal H$.
 \item For every $g\in\mathcal H\setminus\crit(\mathcal H)$, $|g|< k$, and every $x\in V(\struc{A})$ there is an $a\in V(\struc{B})$ such that $g\cup \{x\mapsto a\}\in \mathcal H$.
\end{enumerate}
 If $\crit(\mathcal H)=\emptyset$, then $\mathcal H$ is a \emph{winning strategy}.
\end{definition}
The set $\mathcal H$ is the set of good positions for Duplicator (therefore they are all partial homomorphisms). Non-emptiness and the closure property (2.) ensure that $\mathcal H$ contains the start position $\emptyset$. 
Furthermore, the closure property guarantees that the current position remains a good position for Duplicator when Spoiler picks up pebbles. 
The extension property (3.) ensures that, from every non-critical position, Duplicator has an appropriate answer if Spoiler puts a free pebble on $x$. 
It follows that if there are no critical positions, then Duplicator can always answer accordingly and thus wins the game. 
Otherwise, if Spoiler reaches a critical position, then Duplicator may not have an appropriate answer and the game reaches a critical state. 
In the next lemma we describe how to use critical strategies to prove lower bounds on the number of rounds. 
\begin{lemma}\label{lem:SequenceCriticalStrategies}
	If $\mathcal H_1,\ldots,\mathcal H_l$ is a sequence of critical strategies on the same pair of structures and for all $i<l$ and all $p\in\crit(\mathcal H_i)$ it holds that $p\in \mathcal{H}_{j}\setminus \crit(\mathcal{H}_{j})$ for some $j\leq i+1$, then Duplicator wins the $l$-round existential $k$-pebble game.
\end{lemma}
\begin{proof}
   Starting with $i=1$, Duplicator answers according to the extension property of $\mathcal H_i$, if the current position $p$ is non-critical in $\mathcal H_i$. Otherwise, $p$ is non-critical in $\mathcal H_{j}$ for some $j\leq i+1$ and Duplicator answers according to the extension property of $\mathcal H_{j}$. This allows Duplicator to survive for at least $l$ rounds.
\qed\end{proof}

The two structures $\struc{A}$ and $\struc{B}$ we construct are vertex colored graphs. They are built out of smaller graphs, called \textit{gadgets}. 
Every gadget $Q$ consists of two graphs $Q_S$ and $Q_D$ for Spoiler's and Duplicator's side, respectively. Hence, $Q_S$ and $Q_D$ will be subgraphs of $\struc{A}$ and $\struc{B}$ in the end. The gadgets contain \textit{boundary vertices}\label{page:boundary}, which are the vertices shared with other gadgets. 
To combine two strategies on two connected gadgets we need to ensure that the strategies agree on the boundary of the gadgets. Formally, let a \textit{boundary function} of a strategy $\mathcal{H}$ on a gadget $Q$ be a mapping $\beta$ from the boundary of $Q_S$ to the boundary of $Q_D$ such that $\beta(z)=h(z)$ for all $h\in\mathcal H$ and all $z$ in the domain of $\beta$ and $h$. We say that two strategies $\mathcal G$ and $\mathcal H$ on gadgets $Q$ and $Q'$ are \textit{connectable}, if their boundary functions agree on the common boundary vertices of $Q$ and $Q'$. 
If $\mathcal G$ and $\mathcal H$ are two connectable critical strategies on gadgets $Q=(Q_S,Q_D)$ and $Q'=(Q'_S,Q'_D)$ it is not hard to see that the
\textit{composition}
$$
\mathcal G \circ \mathcal H = \{g\cup h\mid g\in\mathcal G, h\in\mathcal H\}
$$
 is a critical strategy on $Q_S\cup Q'_S$ and $Q_D\cup Q'_D$ with $\crit(\mathcal G\circ\mathcal H) = \crit(\mathcal G)\cup\crit(\mathcal H)$.
Intuitively, playing according to the strategy $\mathcal G \circ \mathcal H$ on $Q$ and $Q'$ means that Duplicator uses strategy $\mathcal G$ on $Q$ and strategy $\mathcal H$ on $Q'$.

\section{The Construction}
\subsection{Overview of the Construction}
\label{sec:CSP-Prop-kCons}

In this section we prove Theorem \ref{thm:kConsDepthLowerBound} for $k\geq 3$. We let 
$
\kmo \defi k-1 \geq 2
$
and construct two vertex colored graphs $\struc{A}_n$ and $\struc{B}_m$ with $O(n)$ and $O(m)$ vertices such that Spoiler needs $\Omega(n^{\kmo}m^{\kmo})$ rounds to win the existential $(\kmo+1)$-pebble game. We color the vertices of both graphs such that the colors partition the vertex set into independent sets, \ie{} every vertex gets one color and there is no edge between vertices of the same color. The basic building blocks in our construction are sets of vertices which allow to store $n^{\kmo}m^{\kmo}$ partial homomorphisms with $\kmo$ pebbles. 

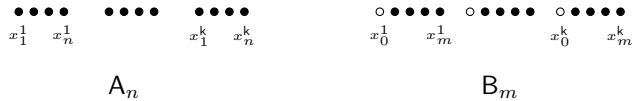
\begin{figure*}[htp]
\begin{center}
\begin{tikzpicture}
      [scale=1, transform shape, knoten/.style={circle,draw=black,fill=black,
      inner  sep=0pt,minimum  size=1mm},kknoten/.style={circle,draw=black,fill=black,
      inner  sep=0pt,minimum  size=0.8mm}, leerknoten/.style={circle,draw=black,fill=white,
      inner  sep=0pt,minimum  size=1mm},leerkknoten/.style={circle,draw=black,fill=white,
      inner  sep=0pt,minimum  size=0.8mm}]

\begin{scope}[xshift=-5cm]
  \node[knoten] (x11) at (0.2,0) [label=below:{\tiny{$x^{1}_1$}}] {}; 
  \node[knoten] (x12) at (0.4,0)  {};
  \node[knoten] (x13) at (0.6,0) {};
  \node[knoten] (x14) at (0.8,0) [label=below:{\tiny{$x^{1}_n$}}] {};

  \node[knoten] (x21) at (1.4,0)  {};
  \node[knoten] (x22) at (1.6,0)  {};
  \node[knoten] (x23) at (1.8,0)  {};
  \node[knoten] (x24) at (2,0)  {};

  \node[knoten] (x31) at (2.6,0)[label=below:{\tiny{$x^{\kmo}_1$}}]{};
  \node[knoten] (x32) at (2.8,0)  {};
  \node[knoten] (x33) at (3,0)  {};
  \node[knoten] (x34) at (3.2,0) [label=below:{\tiny{$x^{\kmo}_n$}}] {};

\end{scope}

  \node[leerknoten] (x10) at (0,0) [label=below:{\tiny{$x^{1}_0$}}] {};
  \node[knoten] (x11) at (0.2,0)  {};
  \node[knoten] (x12) at (0.4,0)  {};
  \node[knoten] (x13) at (0.6,0) {};
  \node[knoten] (x14) at (0.8,0) [label=below:{\tiny{$x^{1}_m$}}] {};

  \node[leerknoten] (x20) at (1.2,0) {};
  \node[knoten] (x21) at (1.4,0)  {};
  \node[knoten] (x22) at (1.6,0)  {};
  \node[knoten] (x23) at (1.8,0)  {};
  \node[knoten] (x24) at (2,0)  {};

  \node[leerknoten] (x30) at (2.4,0) [label=below:{\tiny{$x^{\kmo}_0$}}] {};
  \node[knoten] (x31) at (2.6,0){};
  \node[knoten] (x32) at (2.8,0)  {};
  \node[knoten] (x33) at (3,0)  {};
  \node[knoten] (x34) at (3.2,0) [label=below:{\tiny{$x^{\kmo}_m$}}] {};

\node at (-3.4,-1){$\struc{A}_n$};
\node at (1.6,-1) {$\struc{B}_m$};

\end{tikzpicture}
\end{center}
\caption{Basic vertex blocks. Two vertices $x^i_j$ and $x^{i'}_{j'}$ get the same color iff $i=i'$.} \label{fig:basicblocks}
\end{figure*}

 We introduce vertices $x^i_j$ ($i\in[\kmo]$, $j\in[n]$) in $\struc{A}_n$ and vertices $x^i_j$ ($i\in[\kmo]$, $j\in[m]\cup\{0\}$) in $\struc{B}_m$. For every $i\in[\kmo]$ the vertices $x^i_j$ form a \emph{block} and are colored with the same color (say $P_{x^i}$), which is different from any other color in the entire construction. The vertices $x^i_0$ in structure $\struc{B}_m$ play a special role in our construction and are visualized by $\whitenode$ instead of $\blacknode$ in the pictures. However, they are colored with the same color $P_{x^i}$ as the other vertices $x^i_j$. 
 Because of the coloring, Duplicator has to answer with some $x^i_{j'}$ whenever Spoiler pebbles a vertex $x^i_j$. 
Since there are $nm$ positions for one pebble pair on \blacknode vertices in one block, we get $n^{\kmo}m^{\kmo}$ positions if every block has exactly one pebble pair on \blacknode vertices. 
The \whitenode vertices are used by Duplicator whenever Spoiler does not play the intended way. That is, if Spoiler pebbles a vertex in block $i$ that he is not supposed to pebble now, then Duplicator answers with $x^i_0$. The construction will have the property that this is always a good situation for Duplicator. 

To describe pebble positions on such vertex blocks, we define mappings $\posa\colon{}[\kmo]\to[n]$ and $\posb\colon{}[\kmo]\to[m]$ and call the pebble position $\{(x^i_{\posa(i)},x^i_{\posb(i)})\mid i\in [\kmo]\}$ \emph{valid}.
\spaceconstraints{
					\footnote{At this point we want to stress that the vertex ``$x^i_j$'' may occur in both structures. However, it will become clear from the context which vertex is meant and with a slight abuse of notation we alway assume that $x^i_{\posa(i)}\in V(\struc{A}_n)$ and $x^i_{\posb(i)}\in V(\struc{B}_n)$.} 
}
If such valid position is on the board, then Duplicator answers with $x^i_{\posb(i)}$ if Spoiler pebbles $x^i_{\posa(i)}$ and with $x^i_{0}$ if Spoiler pebbles $x^i_{j}$ for some $j\neq\posa(i)$.
We also need to name positions where Duplicator answers with $x^i_0$ for every vertex in block $i$ and let $T$ be the set of blocks where this happens. 
For $\posa\colon{}[\kmo]\to[n]$, $\posb\colon{}[\kmo]\to[m]$ and $T\subseteq [\kmo]$ we call $\posq=(\posa,\posb,T)$ a \emph{configuration}.
The configuration $\posq$ is \emph{valid} if $T=\emptyset$ and \emph{invalid} otherwise. 
For every configuration $\posq$ and a set of $x^i_j$ vertices as in Figure \ref{fig:basicblocks} we define the  following homomorphism that describes Duplicator's behavior:
$$ \label{page:definition_hq}
h^x_{\posq}(x^i_j) = \begin{cases}
	x^i_{\posb(i)}\text{, if }j=\posa(i)\text{ and }i\notin T, \\
	x^i_{0}\text{, otherwise.}
\end{cases}
$$ 
By $h^x_{\boldsymbol{0}}$ we denote the homomorphism $h^x_{\boldsymbol{0}}(x^i_j)\defi x^i_0$ for all $i\in [\kmo], j\in[n]$.
\label{page:def_h0}
We say that a position of (at most $\kmo+1$) pebble pairs on these vertices is \emph{invalid} if it is a subset of $h^x_{\posq}$ for some invalid configuration $\posq$. 
For valid configurations $\posq=(\posa,\posb,\emptyset)$ we say ``$\posq$ on $x$'' to name the valid pebble position $\{(x^i_{\posa(i)},x^i_{\posb(i)})\mid i\in [\kmo]\}$.
Note that valid pebble positions are not invalid.\footnote{There are pebble positions on the $x^i_j$ vertices that are neither valid nor invalid. However, such positions will not occur in our strategies.} 

In the entire construction there is one unique copy of the $x^i_j$-vertices, which are denoted by $\xnode^i_j$.
Our goal is to force Spoiler to pebble every valid position on $\xnode$ before he wins the game. He is supposed to do so in a specific predefined order. 
To fix this order we define a bijection $\alpha$ between valid configurations $(\posa,\posb,\emptyset)$ and the numbers $0,\ldots,n^{\kmo}m^{\kmo}-1$:
 \begin{align*}
 	\alpha(\posq)&\defi m^{\kmo}\sum^{\kmo}_{i=1} (\posa(i)-1) n^{\kmo-i} + \sum^{\kmo}_{i=1} (\posb(i)-1) m^{\kmo-i}. 
 \end{align*} 
Thus, $\alpha(\posq)$ is the rank of the tuple $(\posa(1),\ldots,\posa(\kmo),\posb(1),\ldots,\posb(\kmo))$ in lexicographical order. If $\alpha(\posq)<n^\kmo m^\kmo-1$, we define the successor $\posq^+=(\posa^+,\posb^+,\emptyset)$ to be the unique valid configuration satisfying $\alpha(\posq^+)=\alpha(\posq)+1$.
In the sequel we introduce gadgets to make sure that:
\begin{itemize}
    \item Spoiler can reach the position $\alpha^{-1}(0)$ on $\xnode$ from $\emptyset$,
	\item Spoiler can reach $\alpha^{-1}(i+1)$ on $\xnode$ from $\alpha^{-1}(i)$ on $\xnode$ and
	\item Spoiler wins from $\alpha^{-1}(n^{\kmo}m^{\kmo}-1)$ on $\xnode$.
\end{itemize}
If we have these properties, we know that Spoiler has a winning strategy in the $(\kmo+1)$-pebble game. To show that Spoiler needs at least $n^{\kmo}m^{\kmo}$ rounds we argue that this is essentially the only way for Spoiler to win the game.

\begin{figure*}[htp]
 \centering
\input{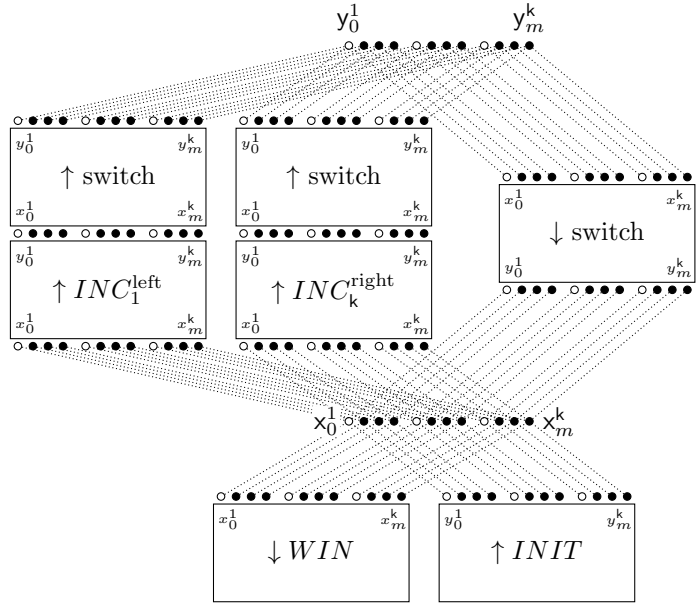}
 \caption{The graph $\struc{B}_m$.
  The boundaries of the gadgets are connected as indicated by the dotted lines (which need to be contracted). The arrows point from the input to the output vertices of the gadgets. } 
 \label{fig:CSP-Prop-kCons-overview}
\end{figure*}

We start with an overview of the gadgets and how they are glued together to form the structures $\struc{A}_n$ and $\struc{B}_m$. 
The boundary
of our gadgets consists of \emph{input} vertices and \emph{output} vertices. For every gadget the set of input (output) vertices is a copy of the vertex set in Figure~\ref{fig:basicblocks} and we write $x^i_j$ ($y^i_j$) to name them. This enables us to glue together the gadgets at their input and output vertices. The overall construction for the graph $\struc{B_m}$ is shown in Figure~\ref{fig:CSP-Prop-kCons-overview}. The schema for $\struc{A}_n$ is similar, it contains Spoiler's side of the corresponding gadgets which are glued together the same way as in $\struc{B}_m$ (just replace $m$ by $n$ and drop the \whitenode vertices).
There are four types of gadgets: the initialization gadget, the winning gadget, several increment gadgets and the switch. 

The \emph{initialization gadget} ensures that Spoiler can reach $\alpha^{-1}(0)$ on $\xnode$, \ie{} the pebble position $\{(\xnode_1^1,\xnode_1^1),\ldots,(\xnode^k_1,\xnode^k_1)\}$. 
This gadget has only output boundary vertices and is used by Spoiler at the beginning of the game.
There are \emph{increment gadgets} $\INC{}^{\text{left}}_i$ and $\INC{}^{\text{right}}_i$ for all $i\in[\kmo]$. 
The input vertices of every increment gadget are identified with the $\xnode$ vertices as depicted in Figure~\ref{fig:CSP-Prop-kCons-overview}.
The increment gadgets (all together) ensure that Spoiler can increment a configuration. 
More precisely, for every valid configuration $\posq$ with $\alpha(\posq) < n^\kmo m^\kmo-1$, there is one increment gadget $\INC{}$ such that Spoiler can reach $\posq^+$ on the output of $\INC{}$ from $\posq$ on the input. 
Every increment gadget is followed by a copy of the \emph{switch}. 
The input of $2\kmo$ switches is identified with the output of the $2\kmo$ increment gadgets and the output of these switches is identified with a unique block of $\ynode$-vertices and the input of one additional \emph{single switch} (see Figure~\ref{fig:CSP-Prop-kCons-overview}). 
The output of this switch is in turn identified with the unique block of $\xnode$-vertices.
The switches are used to perform the transition in the game from $\alpha^{-1}(i)$ on $\xnode$ to $\alpha^{-1}(i+1)$ on $\xnode$. 
Spoiler can pebble a valid position through one switch: from $\posq$ on the input of a switch Spoiler can reach $\posq$ on the output of that switch. 
Hence, Spoiler can simply pebble the incremented position $\alpha^{-1}(i+1)$ from the output of an increment gadget through two switches to the $\xnode$-block. 

Finally, the \emph{winning gadget} ensures that 
from $\alpha^{-1}(n^{\kmo}m^{\kmo}-1)$ on $\xnode$
Spoiler wins the game.
The winning gadget has only input vertices, which are identified with the $\xnode$-vertices. 
From $\alpha^{-1}(n^{\kmo}m^{\kmo}-1)$ on the input, Spoiler can win the game by playing on this gadget. 
On the other hand, the gadget ensures that Spoiler can \emph{only} win from $\alpha^{-1}(n^{\kmo}m^{\kmo}-1)$ on $\xnode$ and Duplicator does not lose from any other configuration on $\xnode$. 

\subsection{The Gadgets}
\label{sec:CSP-Prop-Gadgets}
We now describe the winning gadget and the increment gadgets in detail and provide strategies for Spoiler and Duplicator on them. 
Afterwards we briefly discuss the switch and the initialization gadget.
In the next section we combine the partial strategies on the gadgets to prove Theorem~\ref{thm:kConsDepthLowerBound}. 

The \emph{winning gadget} is shown in Figure~\ref{fig:winning_gadget}. On Spoiler's side there is just one additional vertex $a$, which is connected to $x^i_n$ for all $i\in[\kmo]$. On Duplicator's side there are $\kmo$ additional vertices $a^i, i\in[\kmo]$. Every $a^i$ is connected to all input vertices except $x^i_m$. We use one new vertex color to color the vertex $a$ and all vertices $a_i$. From the position $\{(x_n^1,x_m^1),\ldots,(x^k_n,x^k_m)\}$ ``$\alpha^{-1}(n^{\kmo}m^{\kmo}-1)$ on $x$'' Spoiler wins the game by placing the ($\kmo+1$)st pebble on $a$. Duplicator has to answer with some $a_i$ (because of the coloring). Since there is an edge between $x^i_n$ and $a$ in $\WIN{}_S$ but none between $x^i_m$ and $a_i$ in $\WIN{}_D$, Spoiler wins immediately. 
It is also not hard to see that for any other position where at least one pebble pair $(x_n^j,x_m^j)$ is missing Duplicator can survive by choosing $a_j$. 

\spaceconstraints{
			\begin{lemma}\label{lem:Win_Gadget_Spoiler}
				Spoiler wins the existential $(\kmo+1)$-pebble game on $\WIN{}$ from $\alpha^{-1}(n^{\kmo}m^{\kmo}-1)$ on the input.
			\end{lemma}
			Let $\posq_{\text{win}}=\alpha^{-1}(n^{\kmo}m^{\kmo}-1)$ be the configuration $ (\posa_{\text{win}},\posb_{\text{win}},\emptyset)$ with $\posa_{\text{win}}(i)\defi n$ and $\posb_{\text{win}}(i)\defi m$ for all $i\in[\kmo]$. The next lemma formalizes Duplicator's strategy.

			\begin{lemma}\label{lem:Win_Gadget_Duplicator}
				For every configuration $\posq\neq \posq_{\text{win}}$ there is a winning strategy for Duplicator in the existential $(\kmo+1)$-pebble game on $\WIN{}$ with boundary function $h^x_{\posq}$.\footnote{Recall the definition of $h^x_{\posq}$ on page \pageref{page:definition_hq}.}
			\end{lemma}
			\begin{proof}
				Let $\posq = (\posa,\posb,T)$. Since $\posq \neq \posq_{\text{win}}$ there is an index $j\in[\kmo]$ such that $\posa(j)\neq n$ or $\posb(j)\neq m$ or $j\in T$. Hence, $h^x_{\posq}(x^j_n)\neq x^j_m$ by the definition of $h^x_{\posq}$. Let $h\defi h^x_{\posq}\cup \{(a,a_j)\}$. Note that $h$ is a homomorphism from $\WIN{}_S$ to $\WIN{}_D$ since it preserves vertex colors and maps edges to edges. The winning strategy for Duplicator is $\wp(h)$ which has $h^x_{\posq}$ as boundary function by definition.
			\end{proof}
}
\begin{figure*}
\begin{center}
\begin{tikzpicture}
      [scale=1.4,  knoten/.style={circle,draw=black,fill=black,
      inner  sep=0pt,minimum  size=1mm},kknoten/.style={circle,draw=black,fill=black,
      inner  sep=0pt,minimum  size=0.8mm}, leerknoten/.style={circle,draw=black,fill=white,
      inner  sep=0pt,minimum  size=1mm},leerkknoten/.style={circle,draw=black,fill=white,
      inner  sep=0pt,minimum  size=0.8mm}]

\begin{scope}[xshift=-5cm]
  \node[knoten] (x11) at (0.2,0) [label=below:{\tiny{$x^{1}_1$}}] {}; 
  \node[knoten] (x12) at (0.4,0)  {};
  \node[knoten] (x13) at (0.6,0) {};
  \node[knoten] (x14) at (0.8,0) [label=below:{\tiny{$x^{1}_n$}}] {};

  \node[knoten] (x21) at (1.4,0)  {};
  \node[knoten] (x22) at (1.6,0)  {};
  \node[knoten] (x23) at (1.8,0)  {};
  \node[knoten] (x24) at (2,0)  {};

  \node[knoten] (x31) at (2.6,0)[label=below:{\tiny{$x^{\kmo}_1$}}]{};
  \node[knoten] (x32) at (2.8,0)  {};
  \node[knoten] (x33) at (3,0)  {};
  \node[knoten] (x34) at (3.2,0) [label=below:{\tiny{$x^{\kmo}_n$}}] {};

  \node[knoten] (a) at (1.7,1)[label=right:{{$a$}}]  {};
  \draw[-] (a) -- (x14);
  \draw[-] (a) -- (x24);
  \draw[-] (a) -- (x34);

\end{scope}

  \node[leerknoten] (x10) at (0,0) [label=below:{\tiny{$x^{1}_0$}}] {};
  \node[knoten] (x11) at (0.2,0)  {};
  \node[knoten] (x12) at (0.4,0)  {};
  \node[knoten] (x13) at (0.6,0) {};
  \node[knoten] (x14) at (0.8,0) [label=below:{\tiny{$x^{1}_m$}}] {};

  \node[leerknoten] (x20) at (1.2,0) {};
  \node[knoten] (x21) at (1.4,0)  {};
  \node[knoten] (x22) at (1.6,0)  {};
  \node[knoten] (x23) at (1.8,0)  {};
  \node[knoten] (x24) at (2,0)  {};

  \node[leerknoten] (x30) at (2.4,0) [label=below:{\tiny{$x^{\kmo}_0$}}] {};
  \node[knoten] (x31) at (2.6,0){};
  \node[knoten] (x32) at (2.8,0)  {};
  \node[knoten] (x33) at (3,0)  {};
  \node[knoten] (x34) at (3.2,0) [label=below:{\tiny{$x^{\kmo}_m$}}] {};

  \node[knoten] (a1) at (0.5,1)[label=right:{{$a_1$}}]  {};
  \node[knoten] (a2) at (1.7,1)[label=right:{{$a_i$}}]  {};
  \node[knoten] (a3) at (2.9,1)[label=right:{{$a_\kmo$}}]  {};
  
  \foreach \i/\a in {2/1,3/1,1/2,3/2,1/3,2/3} \foreach \j in {0,1,2,3,4} \draw[-] (a\a) -- (x\i\j);
  \foreach \a in {1,2,3} \foreach \j in {0,1,2,3} \draw[-] (a\a) -- (x\a\j);

\node at (-3.4,-0.5){$\WIN{}_S\subseteq\struc{A}_n$};
\node at (1.6,-0.5) {$\WIN{}_D\subseteq\struc{B}_m$};

\end{tikzpicture}
\end{center}
\caption{The winning gadget.} \label{fig:winning_gadget}
\end{figure*}
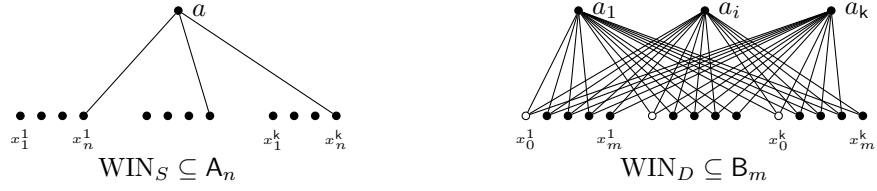

The \emph{increment gadgets} enable Spoiler to reach the successor $\posq^+$ from $\posq$. 
Recall that we identify every valid configuration $\posq = (\posa,\posb,\emptyset)$ with the tuple $(\posa(1),\ldots,\posa(\kmo),\posb(1),\ldots,\posb(\kmo)) \in [n]^k\times[m]^k$ and define $\alpha(\posq)$ to be the rank (from $0$ to $n^km^k-1$) of this tuple in lexicographical order. 
Let $\posq$ be a valid configuration with $\alpha(\posq)<n^\kmo m^\kmo - 1$ and successor $\posq^+=(\posa^+(1),\ldots,\posa^+(\kmo),\posb^+(1),\ldots,\posb^+(\kmo))$. 
We use two types of increment gadgets, \emph{left} and \emph{right}, depending on whether the left-hand side of the tuple changes after incrementation or not.
There are $\kmo$ increment gadgets of each type. Spoiler uses them depending on which position the last carryover occurs. If
\begin{align*}
	\posq  &= (\posa(1),\ldots,\posa(\kmo),&&\posb(1),\ldots,\posb(\ell-1),
	&&\posb(\ell)<m, &&m,\ldots,m)\text{ and hence } \\
	\posq^+ &= (\posa(1),\ldots,\posa(\kmo),&&\posb(1),\ldots,\posb(\ell-1),
	&&\posb(\ell)+1, &&1,\ldots,1), 
\end{align*}
then Spoiler uses the increment gadget $\INC{}^{\text{right}}_\ell$ to reach $\posq^+$ on the output from $\posq$ on the input. 
If
\begin{align*}
	\posq  &= (\posa(1),\ldots,\posa(\ell-1),
	&&\posa(\ell)<n, &&n,\ldots,n,&&m,\ldots,m)\text{ and hence } \\
	\posq^+  &= (\posa(1),\ldots,\posa(\ell-1),
	&&\posa(\ell)+1, &&1,\ldots,1,&&1,\ldots,1), 
\end{align*}
then Spoiler uses $\INC{}^{\text{left}}_\ell$. Thus, for every valid configuration $\posq$ with $\alpha(\posq)<n^\kmo m^\kmo-1$ there is exactly one \emph{applicable} increment gadget.
\spaceconstraints{
			The next definition formalizes this.
			\begin{definition}\label{def:applicable_inc_gadget}
			 	The right increment gadget $\INC{}^{\text{right}}_\ell$ is \emph{applicable} to a configuration $\posq$ if $\posq=(\posa,\posb,\emptyset)$ is valid, $\posb(\ell)<m$ and $\posb(i)=m$ for all $i>\ell$. A left increment gadget $\INC{}^{\text{left}}_\ell$ is applicable to $\posq$ if $\posq=(\posa,\posb,\emptyset)$ is valid, $\posa(\ell)<n$, $\posa(i)=n$ for all $i>\ell$ and $\posb(i)=m$ for all $i\in[\kmo]$. 
			 \end{definition} 
			It follows from the definition that for every valid configuration $\posq$ with $\alpha(\posq)<n^\kmo m^\kmo-1$ there is exactly one increment gadget that is applicable to $\posq$. 
			Furthermore if there is some applicable increment gadget for $\posq$, then $\alpha(\posq)<n^\kmo m^\kmo-1$ and $\posq^+$ is defined.
			We define $T^{\text{right}}_\ell(\posq)\subseteq[\kmo]$ and $T^{\text{left}}_\ell(\posq)\subseteq[\kmo]$ to be the set of blocks that contradict the applicability of $\INC{}^{\text{\upshape right}}_\ell$ ($\INC{}^{\text{\upshape left}}_\ell$, resp.) to $\posq$. 
			\label{page:def_T_right}
			That is, $i\in T^{\text{right}}_\ell(\posq)$ for a configuration $\posq=(\posa,\posb,T)$ if one of the following conditions is satisfied:
			\begin{itemize}
				\item $i=\ell$ and $\posb(i)=m$,
				\item $i>\ell$ and $\posb(i)\neq m$,
				\item $i\in T$.
			\end{itemize}
			Similarly, $i\in T^{\text{left}}_\ell(\posq)$ if
			\label{page:def_T_left}
			\begin{itemize}
				\item $i=\ell$ and $\posa(\ell)=n$,
				\item $i>\ell$ and $\posa(i)\neq n$,
				\item $\posb(i)\neq m$ or
				\item $i\in T$.
			\end{itemize}
			Therefore, $T^{\text{right}}_\ell(\posq)=\emptyset$ ($T^{\text{left}}_\ell(\posq)=\emptyset$) if and only if $\INC{}^{\text{\upshape right}}_\ell$ ($\INC{}^{\text{\upshape left}}_\ell$) is applicable to $\posq$.
			We now start describing the gadgets. Every increment gadget contains input vertices $x^i_j$, output vertices $y^i_j$ and no further vertices.
}
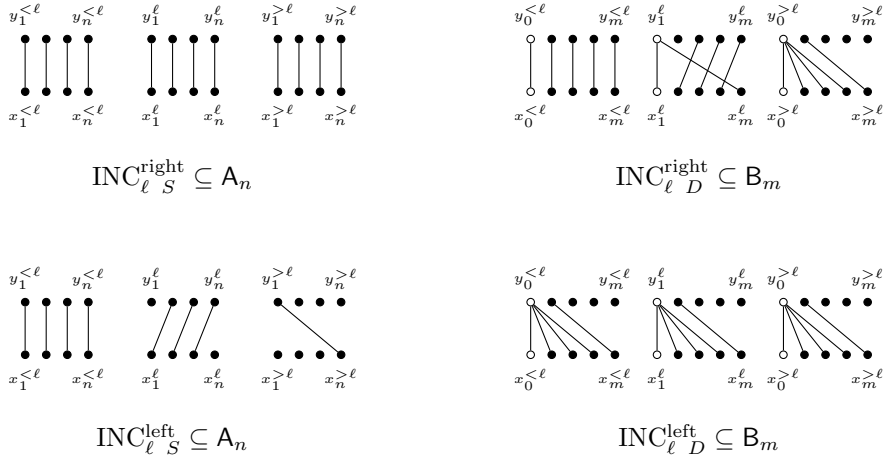
\begin{figure*}
\begin{center}
\begin{tikzpicture}
      [scale=1.4,  knoten/.style={circle,draw=black,fill=black,
      inner  sep=0pt,minimum  size=1mm},kknoten/.style={circle,draw=black,fill=black,
      inner  sep=0pt,minimum  size=0.8mm}, leerknoten/.style={circle,draw=black,fill=white,
      inner  sep=0pt,minimum  size=1mm},leerkknoten/.style={circle,draw=black,fill=white,
      inner  sep=0pt,minimum  size=0.8mm}]

i
\begin{scope}[xshift=-5cm]
  \node[knoten] (x11) at (0.2,0) [label=below:{\tiny{$x^{<\ell}_1$}}] {}; 
  \node[knoten] (x12) at (0.4,0)  {};
  \node[knoten] (x13) at (0.6,0) {};
  \node[knoten] (x14) at (0.8,0) [label=below:{\tiny{$x^{<\ell}_n$}}] {};

  \node[knoten] (x21) at (1.4,0)  [label=below:{\tiny{$x^{\ell}_1$}}]{};
  \node[knoten] (x22) at (1.6,0)  {};
  \node[knoten] (x23) at (1.8,0)  {};
  \node[knoten] (x24) at (2,0)  [label=below:{\tiny{$x^{\ell}_n$}}]{};

  \node[knoten] (x31) at (2.6,0)[label=below:{\tiny{$x^{>\ell}_1$}}]{};
  \node[knoten] (x32) at (2.8,0)  {};
  \node[knoten] (x33) at (3,0)  {};
  \node[knoten] (x34) at (3.2,0) [label=below:{\tiny{$x^{>\ell}_n$}}] {};

\end{scope}

  \node[leerknoten] (xx10) at (0,0) [label=below:{\tiny{$x^{<\ell}_0$}}] {};
  \node[knoten] (xx11) at (0.2,0)  {};
  \node[knoten] (xx12) at (0.4,0)  {};
  \node[knoten] (xx13) at (0.6,0) {};
  \node[knoten] (xx14) at (0.8,0) [label=below:{\tiny{$x^{<\ell}_m$}}] {};

  \node[leerknoten] (xx20) at (1.2,0) [label=below:{\tiny{$x^{\ell}_1$}}]{};
  \node[knoten] (xx21) at (1.4,0)  {};
  \node[knoten] (xx22) at (1.6,0)  {};
  \node[knoten] (xx23) at (1.8,0)  {};
  \node[knoten] (xx24) at (2,0)  [label=below:{\tiny{$x^{\ell}_m$}}]{};

  \node[leerknoten] (xx30) at (2.4,0) [label=below:{\tiny{$x^{>\ell}_0$}}] {};
  \node[knoten] (xx31) at (2.6,0){};
  \node[knoten] (xx32) at (2.8,0)  {};
  \node[knoten] (xx33) at (3,0)  {};
  \node[knoten] (xx34) at (3.2,0) [label=below:{\tiny{$x^{>\ell}_m$}}] {};

\begin{scope}[yshift = 0.5cm]
\begin{scope}[xshift=-5cm]
  \node[knoten] (y11) at (0.2,0) [label=above:{\tiny{$y^{<\ell}_1$}}] {}; 
  \node[knoten] (y12) at (0.4,0)  {};
  \node[knoten] (y13) at (0.6,0) {};
  \node[knoten] (y14) at (0.8,0) [label=above:{\tiny{$y^{<\ell}_n$}}] {};

  \node[knoten] (y21) at (1.4,0)  [label=above:{\tiny{$y^{\ell}_1$}}]{};
  \node[knoten] (y22) at (1.6,0)  {};
  \node[knoten] (y23) at (1.8,0)  {};
  \node[knoten] (y24) at (2,0)  [label=above:{\tiny{$y^{\ell}_n$}}]{};

  \node[knoten] (y31) at (2.6,0)[label=above:{\tiny{$y^{>\ell}_1$}}]{};
  \node[knoten] (y32) at (2.8,0)  {};
  \node[knoten] (y33) at (3,0)  {};
  \node[knoten] (y34) at (3.2,0) [label=above:{\tiny{$y^{>\ell}_n$}}] {};

\end{scope}

  \node[leerknoten] (yy10) at (0,0) [label=above:{\tiny{$y^{<\ell}_0$}}] {};
  \node[knoten] (yy11) at (0.2,0)  {};
  \node[knoten] (yy12) at (0.4,0)  {};
  \node[knoten] (yy13) at (0.6,0) {};
  \node[knoten] (yy14) at (0.8,0) [label=above:{\tiny{$y^{<\ell}_m$}}] {};

  \node[leerknoten] (yy20) at (1.2,0) [label=above:{\tiny{$y^{\ell}_1$}}]{};
  \node[knoten] (yy21) at (1.4,0)  {};
  \node[knoten] (yy22) at (1.6,0)  {};
  \node[knoten] (yy23) at (1.8,0)  {};
  \node[knoten] (yy24) at (2,0)  [label=above:{\tiny{$y^{\ell}_m$}}]{};

  \node[leerknoten] (yy30) at (2.4,0) [label=above:{\tiny{$y^{>\ell}_0$}}] {};
  \node[knoten] (yy31) at (2.6,0){};
  \node[knoten] (yy32) at (2.8,0)  {};
  \node[knoten] (yy33) at (3,0)  {};
  \node[knoten] (yy34) at (3.2,0) [label=above:{\tiny{$y^{>\ell}_m$}}] {};
	
\end{scope}

\foreach \i in {1,2,3} \foreach \j in {1,2,3,4} \draw[-] (x\i\j) -- (y\i\j);
\foreach \i in {1} \foreach \j in {0,1,2,3,4} \draw[-] (xx\i\j) -- (yy\i\j);
\foreach \oben/\unten in {0/4,3/2,2/1,4/3,0/0} \draw[-] (xx2\unten) -- (yy2\oben);
\foreach \oben/\unten in {0/0,0/2,0/3,0/1,1/4} \draw[-] (xx3\unten) -- (yy3\oben);

\node at (-3.4,-0.8){${\INC{}^{\text{right}}_{\ell\;\; S}}\subseteq\struc{A}_n$};
\node at (1.6,-0.8) {${\INC{}^{\text{right}}_{\ell\;\; D}}\subseteq\struc{B}_m$};

\begin{scope}[yshift=-2.5cm]

	\begin{scope}[xshift=-5cm]
	  \node[knoten] (x11) at (0.2,0) [label=below:{\tiny{$x^{<\ell}_1$}}] {}; 
	  \node[knoten] (x12) at (0.4,0)  {};
	  \node[knoten] (x13) at (0.6,0) {};
	  \node[knoten] (x14) at (0.8,0) [label=below:{\tiny{$x^{<\ell}_n$}}] {};

	  \node[knoten] (x21) at (1.4,0)  [label=below:{\tiny{$x^{\ell}_1$}}]{};
	  \node[knoten] (x22) at (1.6,0)  {};
	  \node[knoten] (x23) at (1.8,0)  {};
	  \node[knoten] (x24) at (2,0)  [label=below:{\tiny{$x^{\ell}_n$}}]{};

	  \node[knoten] (x31) at (2.6,0)[label=below:{\tiny{$x^{>\ell}_1$}}]{};
	  \node[knoten] (x32) at (2.8,0)  {};
	  \node[knoten] (x33) at (3,0)  {};
	  \node[knoten] (x34) at (3.2,0) [label=below:{\tiny{$x^{>\ell}_n$}}] {};

	\end{scope}

	  \node[leerknoten] (xx10) at (0,0) [label=below:{\tiny{$x^{<\ell}_0$}}] {};
	  \node[knoten] (xx11) at (0.2,0)  {};
	  \node[knoten] (xx12) at (0.4,0)  {};
	  \node[knoten] (xx13) at (0.6,0) {};
	  \node[knoten] (xx14) at (0.8,0) [label=below:{\tiny{$x^{<\ell}_m$}}] {};

	  \node[leerknoten] (xx20) at (1.2,0) [label=below:{\tiny{$x^{\ell}_1$}}]{};
	  \node[knoten] (xx21) at (1.4,0)  {};
	  \node[knoten] (xx22) at (1.6,0)  {};
	  \node[knoten] (xx23) at (1.8,0)  {};
	  \node[knoten] (xx24) at (2,0)  [label=below:{\tiny{$x^{\ell}_m$}}]{};

	  \node[leerknoten] (xx30) at (2.4,0) [label=below:{\tiny{$x^{>\ell}_0$}}] {};
	  \node[knoten] (xx31) at (2.6,0){};
	  \node[knoten] (xx32) at (2.8,0)  {};
	  \node[knoten] (xx33) at (3,0)  {};
	  \node[knoten] (xx34) at (3.2,0) [label=below:{\tiny{$x^{>\ell}_m$}}] {};

	\begin{scope}[yshift = 0.5cm]
	\begin{scope}[xshift=-5cm]
	  \node[knoten] (y11) at (0.2,0) [label=above:{\tiny{$y^{<\ell}_1$}}] {}; 
	  \node[knoten] (y12) at (0.4,0)  {};
	  \node[knoten] (y13) at (0.6,0) {};
	  \node[knoten] (y14) at (0.8,0) [label=above:{\tiny{$y^{<\ell}_n$}}] {};

	  \node[knoten] (y21) at (1.4,0)  [label=above:{\tiny{$y^{\ell}_1$}}]{};
	  \node[knoten] (y22) at (1.6,0)  {};
	  \node[knoten] (y23) at (1.8,0)  {};
	  \node[knoten] (y24) at (2,0)  [label=above:{\tiny{$y^{\ell}_n$}}]{};

	  \node[knoten] (y31) at (2.6,0)[label=above:{\tiny{$y^{>\ell}_1$}}]{};
	  \node[knoten] (y32) at (2.8,0)  {};
	  \node[knoten] (y33) at (3,0)  {};
	  \node[knoten] (y34) at (3.2,0) [label=above:{\tiny{$y^{>\ell}_n$}}] {};

	\end{scope}

	  \node[leerknoten] (yy10) at (0,0) [label=above:{\tiny{$y^{<\ell}_0$}}] {};
	  \node[knoten] (yy11) at (0.2,0)  {};
	  \node[knoten] (yy12) at (0.4,0)  {};
	  \node[knoten] (yy13) at (0.6,0) {};
	  \node[knoten] (yy14) at (0.8,0) [label=above:{\tiny{$y^{<\ell}_m$}}] {};

	  \node[leerknoten] (yy20) at (1.2,0) [label=above:{\tiny{$y^{\ell}_1$}}]{};
	  \node[knoten] (yy21) at (1.4,0)  {};
	  \node[knoten] (yy22) at (1.6,0)  {};
	  \node[knoten] (yy23) at (1.8,0)  {};
	  \node[knoten] (yy24) at (2,0)  [label=above:{\tiny{$y^{\ell}_m$}}]{};

	  \node[leerknoten] (yy30) at (2.4,0) [label=above:{\tiny{$y^{>\ell}_0$}}] {};
	  \node[knoten] (yy31) at (2.6,0){};
	  \node[knoten] (yy32) at (2.8,0)  {};
	  \node[knoten] (yy33) at (3,0)  {};
	  \node[knoten] (yy34) at (3.2,0) [label=above:{\tiny{$y^{>\ell}_m$}}] {};
		
	\end{scope}

	\foreach \z in {1,2,3} \foreach \oben/\unten in {0/0,0/2,0/3,0/1,1/4} \draw[-] (xx\z\unten) -- (yy\z\oben);

	\foreach \i in {1} \foreach \j in {1,2,3,4} \draw[-] (x\i\j) -- (y\i\j);
	\foreach \oben/\unten in {2/1,3/2,4/3} \draw[-] (x2\unten) -- (y2\oben);
	\draw[-] (x34) -- (y31);

	\node at (-3.4,-0.8){${\INC{}^{\text{left}}_{\ell\;\; S}}\subseteq\struc{A}_n$};
	\node at (1.6,-0.8) {${\INC{}^{\text{left}}_{\ell\;\; D}}\subseteq\struc{B}_m$};

\end{scope}

\end{tikzpicture}
\end{center}
\caption{The increment gadgets.} \label{fig:inc}
\end{figure*}
The increment gadgets are shown in Figure~\ref{fig:inc}. All input vertices $x^i_j$ have at most one output vertex $y^i_{j'}$ as neighbor. 
Furthermore, if the gadget is applicable to a valid configuration $\posq=(\posa,\posb,\emptyset)$, then the unique neighbor of $x^i_{\posa(i)}$ is $y^i_{\posa^+(i)}$ and the unique neighbor of $x^i_{\posb(i)}$ is $y^i_{\posb^+(i)}$. This enables Spoiler to reach $\posq^+$ on the output from $\posq$ on the input by the following procedure.
First, Spoiler places the remaining pebble on $y^1_{\posa^+(1)}$. Since this vertex is adjacent to $x^1_{\posa(1)}$, Duplicator has to answer with $y^1_{\posb^+(1)}$, the only vertex that is adjacent to $x^1_{\posb(1)}$. 
Afterwards, Spoiler picks up the pebble pair from $(x^1_{\posa(1)},x^1_{\posb(1)})$. 
On the second block Spoiler proceeds the same way: he pebbles $y^2_{\posa^+(2)}$, forces the position $(y^2_{\posa^+(2)},y^2_{\posb^+(2)})$ and picks up the pebbles from $(x^2_{\posa(2)},x^2_{\posb(2)})$. By iterating this procedure Spoiler reaches $\posq^+$ on the output.

If Spoiler tries to move a configuration through one increment gadget that is \emph{not} applicable, then Duplicator can answer with an invalid configuration on the output as follows. On the one hand, if the gadget is not applicable because some $\posb(i)$ does not have the specified value, then $x^i_{\posb(i)}$ is adjacent to $y^i_0$. On the other hand, if some $\posa(i)$ has the wrong value, then $x^i_{\posa(i)}$ is not adjacent to an output vertex. In both cases Duplicator can safely pebble $y^i_0$ if Spoiler queries some $y^i_{j}$ and hence maintain an invalid output position.
The next lemma provides these strategies, a formal proof is given in the full version of the paper.
\begin{lemma}\label{lem:inc}
	Let $\posq=(\posa,\posb,T)$ be a configuration and $\INC$ an increment gadget.
	\begin{itemize}
		\item[1.] If $\INC$ is applicable to $\posq$, then Spoiler can reach $\posq^+$ on the output from $\posq$ on the input of $\INC$.
		\item[2.] If $\INC$ is applicable to $\posq$, then there is a winning strategy for Duplicator with boundary function $h^x_{\posq}$ on the input and $h^y_{\posq^+}$ on the output.
		\item[3.] If $\INC$ is not applicable to $\posq$, then there is a winning strategy for Duplicator with boundary function $h^x_{\posq}$ on the input and $h^y_{\posq_{\text{\upshape inv}}}$ on the output for an invalid configuration 
		$\posq_{\text{\upshape inv}}$.
	\end{itemize}
\end{lemma}

\spaceconstraints{

					While we describe the gadget it might be useful to have the situation in mind when we want to use this gadget.
					Hence, assume that the current pebble position is 
					$$\{(x^1_{\posa(1)},x^1_{\posb(1)}),\ldots,(x^\ell_{\posa(\ell)},x^\ell_{\posb(\ell)}),(x^{\ell+1}_{\posa(\ell+1)},x^{\ell+1}_{m}),\ldots,(x^{\kmo}_{\posa(\kmo)},x^{\kmo}_{m})\}$$ 
					and Spoiler wants to reach the incremented position
					$$\{(y^1_{\posa(1)},y^1_{\posb(1)}),\ldots,(y^\ell_{\posa(\ell)},y^\ell_{\posb(\ell)+1}),(y^{\ell+1}_{\posa(\ell+1)},y^{\ell+1}_{1}),\ldots,(y^{\kmo}_{\posa(\kmo)},y^{\kmo}_{1})\}.$$
					In Spoiler's side of the gadget all blocks have the same shape: every $x^i_j$ is connected to $y^i_j$. In Duplicator's side the first $\ell-1$ blocks look the same, $x^i_j$ is connected to $y^i_j$ for all $i<\ell$ and $0\leq j\leq m$. This ensures that Spoiler can reach $(y^i_{\posa(i)},y^{i}_{\posb(i)})$ from $(x^i_{\posa(i)},x^{i}_{\posb(i)})$, for all $i<\ell$, by placing the remaining $(\kmo+1)$st pebble on $y^i_{\posa(i)}$.  
					Block $\ell$ on Duplicator's side looks different (this is the block where we want to increment $\posb(\ell)$ to $\posb(\ell)+1$). 
					The input vertex $x^\ell_i$ is connected to $y^\ell_{i+1}$ for every $i\in[m-1]$. Furthermore, we connect $x^\ell_m$ to the special vertex $y^\ell_0$. From the position $(x^\ell_{\posa(\ell)},x^\ell_{\posb(\ell)})$, where $\posb(\ell)<m$, Spoiler can reach $(y^\ell_{\posa(\ell)},y^\ell_{\posb(\ell)+1})$ by placing the remaining pebble on $y^\ell_{\posa(\ell)}$. Every other block $i>\ell$ on Duplicator's side contains an edge between $x^i_m$ and $y^i_1$ to ensure that Spoiler can reach $(y^i_{\posa(i)},y^i_{1})$ from $(x^i_{\posa(i)},x^i_{m})$ in the same way. Furthermore, the other input vertices $x^i_j$, $j<m$, are connected to the special vertex $y^i_0$. 
					\begin{lemma}\label{lem:inc_right_Sp}
						Let $\ell\in[\kmo]$ and $\posq$ be a configuration such that $\INC{}^{\text{\upshape right}}_\ell$ is applicable to $\posq$. Then Spoiler can reach $\posq^+$ on the output from $\posq$ on the input of the right increment gadget $\INC{}^{\text{\upshape right}}_\ell$.  
					\end{lemma}
					\begin{proof}
						Let $\posq=(\posa,\posb,\emptyset)$ with successor $\posq^+=(\posa^+,\posb^+,\emptyset)$. 
						Note that $\posa=\posa^+$ by Definition~\ref{def:applicable_inc_gadget}.
						We have to show that Spoiler can reach $\{(y^i_{\posa(i)},y^i_{\posb^+(i)})\mid i\in[\kmo]\}$ from $\{(x^i_{\posa(i)},x^i_{\posb(i)})\mid i\in[\kmo]\}$. 
						By definition of the gadget, $y^i_{\posa(i)}$ is a neighbor of $x^i_{\posa(i)}$ on Spoiler's side. 
						On Duplicator's side, $y^i_{\posb^+(i)}$ is a neighbor of $x^i_{\posb(i)}$. 
						Furthermore, $y^i_{\posb^+(i)}$ is the only neighbor among $y^i_{0},\ldots,y^i_{m}$, which are all the vertices having the same color as $y^i_{\posa(i)}$ (see Figure~\ref{fig:right_increment}). 
						In such a situation Spoiler can reach $\posq^+$ at the output by the following procedure. 
						First, Spoiler places the remaining pebble on $y^1_{\posa(1)}$. Since this vertex is adjacent to $x^1_{\posa(1)}$, Duplicator has to answer with a vertex of the same color (\ie{} $y^1_j$ for some $0\leq j\leq m$) that is adjacent to $x^1_{\posb(1)}$. 
						The only vertex satisfying this property is $y^1_{\posb^+(1)}$. 
						Thus, the new pebble position is  $(y^1_{\posa(1)},y^1_{\posb^+(1)})$ and Spoiler can pick up the pebble pair from $(x^1_{\posa(1)},x^1_{\posb(1)})$. 
						On the second block Spoiler proceeds the same way: he pebbles $y^2_{\posa(2)}$, forces the position $(y^2_{\posa(2)},y^2_{\posb^+(2)})$ and picks up the pebbles from $(x^2_{\posa(2)},x^2_{\posb(2)})$. 
						By iterating this procedure, Spoiler can reach the position $\{(y^i_{\posa(i)},y^i_{\posb^+(i)})\mid i\in[\kmo]\}$.
					\end{proof}
					We have shown that Spoiler can increment a configuration using an applicable right increment gadget. The next step is to show that this is essentially everything Spoiler can do with this gadget. 
					To show this, we have to ensure the following two assertions. 
					First, from a valid configuration on the input of an increment gadget applicable to it Spoiler can \emph{only} reach the incremented configuration on the output. 
					Second, if the gadget is not applicable to the configuration on the input, then Spoiler cannot reach a valid configuration on the output. 
					We provide Duplicator with appropriate counter strategies to ensure these assertions. In general, Duplicator's counter strategy is to pebble a \whitenode vertex whenever she has the possibility to do so. 
					On the one hand, if Spoiler pebbles a vertex that is neither equal nor adjacent to another pebbled vertex in $\struc{A}_n$, then Duplicator can answer with the \whitenode vertex in the corresponding block. 
					On the other hand, if Spoiler pebbles an edge $\{x^i_j,y^i_j\}$, then Duplicator may answer with the \whitenode vertex provided there is also an edge in her graph.
					The next lemma formalizes Duplicator's strategy.
					\begin{lemma}\label{lem:inc_right_Du}
						Let $\ell\in[\kmo]$ and $\posq=(\posa,\posb,T)$ be a configuration.
						\begin{itemize}
							\item[1.] If $\INC{}^{\text{\upshape right}}_\ell$ is applicable to $\posq$, then there is a winning strategy for Duplicator with boundary function $h^x_{\posq}$ on the input and $h^y_{\posq^+}$ on the output.
							\item[2.] If $\INC{}^{\text{\upshape right}}_\ell$ is not applicable to $\posq$, then there is a winning strategy for Duplicator with boundary function $h^x_{\posq}$ on the input and $h^y_{\posq_{\text{\upshape inv}}}$ on the output for an invalid configuration 
							$\posq_{\text{\upshape inv}}$.
						\end{itemize}
					\end{lemma}
					\begin{proof}
						To prove the first statement let $\posq$ be a valid configuration (hence $T=\emptyset$) such that  $\INC{}^{\text{\upshape right}}_\ell$ is applicable to $\posq$ and $\posq^+=(\posa^+,\posb^+,\emptyset)$ be the incremented position. We claim that $h\defi h^x_{\posq}\cup h^y_{\posq^+}$ is a homomorphism from ${\INC{}^{\text{\upshape right}}_{\ell\;\; S}}$ to ${\INC{}^{\text{\upshape right}}_{\ell\;\; D}}$. It follows that $\mathcal H \defi \wp(h)$ is a winning strategy with the desired boundary function.
						Since $h^x_{\posq}$ and $h^y_{\posq^+}$ preserve vertex colors, it remains to verify that all edges were mapped to edges. Hence, we have to check that for all $i\in[\kmo]$ and $j\in[n]$ there is an edge between $h(x^i_j)$ and $h(y^i_j)$ in Duplicator's graph. Recall that $\posa=\posa^+$ whenever a right increment gadget is applicable to $\posq$. By definition
						\begin{align*}		
						h(x^i_j) &=  \begin{cases}
						x^i_{\posb(i)}\text{, if }j=\posa(i), \\
						x^i_{0}\text{, otherwise,}
						\end{cases} &
						h(y^i_j) &=  \begin{cases}
						y^i_{\posb^+(i)}\text{, if }j=\posa(i), \\
						y^i_{0}\text{, otherwise.}
						\end{cases}  
						\end{align*}
						Since there is an edge between \whitenode vertices of the corresponding blocks it follows that $h(x^i_j)$ and $h(y^i_j)$ are adjacent for all $j\neq\posa(i)$. 
						By the choice of $\posq$ and $\posq^+$ we have $\posb^+(i)=\posb(i)$ for all $i<\ell$, $\posb^+(\ell)=\posb(\ell)+1$, $\posb(i)=m$ and $\posb^+(i)=1$ for all $i>\ell$. 
						Thus, by the definition of the gadget, there is an edge between $h(x^i_{\posa(i)})=x^i_{\posb(i)}$ and $h(y^i_{\posa(i)})=y^i_{\posb^+(i)}$.

						For the second statement 
						recall that $T^{\text{right}}_\ell(\posq)$ (defined on page \pageref{page:def_T_right}) is the set of blocks that do not satisfy the applicability condition.
						Assume that $\INC{}^{\text{\upshape right}}_\ell$ is not applicable to $\posq=(\posa,\posb,T)$ and hence $T^{\text{\upshape right}}_\ell(\posq)\neq\emptyset$. Note that on Duplicator's side of the gadget every vertex $x^i_j$ has exactly one neighbor $y^i_{j'}$. 
						We have designed the gadget such that the unique neighbor of $x^i_{\posb(i)}$ is $y^i_0$ if and only if $i\in T^{\text{\upshape right}}_\ell(\posq)$. That is, if block $i$ contradicts the applicability condition, then Duplicator has the chance to move to the \whitenode vertex $y^i_0$ when Spoiler moves upwards. Hence, Duplicator can avoid a valid configuration on the output. Formally, Duplicator's strategy is $\wp(h^x_{\posq}\cup h^y_{\posq_\text{\upshape inv}})$ where $\posq_{\text{\upshape inv}}=(\posa,\posb_{\text{\upshape inv}}, T^{\text{\upshape right}}_\ell(\posq))$ with 
						\begin{align*}		
						\posb_{\text{\upshape inv}}(i) &=  \begin{cases}
						\text{arbitrary, if }i\in T^{\text{\upshape right}}_\ell(\posq), \\
						\text{$j$, such that $y^i_j$ is the neighbor of $x^i_{\posb(i)}$, otherwise.}\mbox{\qed}
						\end{cases}  
						\end{align*}
					\end{proof}
					This concludes the strategies on the right increment gadget. In the remaining part of this paragraph we describe similar strategies for the left increment gadget. The gadget $\INC{}^{\text{left}}_\ell$ is shown in Figure~\ref{fig:left_increment}. It ensures that from a configuration $\posq$ (where $\INC{}^{\text{left}}_\ell$ is applicable) on the input Spoiler can reach the incremented position $\posq^+$ on the output. 
					That is, starting from a position
					$$
					\{(x^1_{\posa(1)},x^1_{m}),
					\ldots,
					(x^\ell_{\posa(\ell)},x^\ell_{m}),
					(x^{\ell+1}_{n},x^{\ell+1}_{m}),
					\ldots,
					(x^{\kmo}_{n},x^{\kmo}_{m})\}$$ 
					Spoiler can reach
					$$\{(y^1_{\posa(1)},y^1_{1}),
					\ldots,
					(y^\ell_{\posa(\ell)+1},y^\ell_{1}),
					(y^{\ell+1}_{1},y^{\ell+1}_{1}),
					\ldots,
					(y^{\kmo}_{1},y^{\kmo}_{1})\}.$$
					To achieve this, there are edges $\{x^i_j,y^i_j\}$ (for $i<\ell,j\in[n]$), $\{x^\ell_j,y^\ell_{j+1}\}$ (for $j\in[n-1]$) and $\{x^i_n,y^i_1\}$ (for $i>\ell$) in Spoiler's graph. In Duplicator's graph every vertex $x^i_m$ is adjacent to $y^i_1$, all other input vertices $x^i_j$, $j<m$, are connected to the \whitenode vertex $y^i_0$. The next lemma describes Spoiler's strategy, which is similar to the strategy on the right increment gadget (Lemma~\ref{lem:inc_right_Sp}). 
					\begin{lemma}\label{lem:inc_left_Sp}
						Let $\ell\in[\kmo]$, $\posq$ be a configuration such that $\INC{}^{\text{\upshape left}}_\ell$ is applicable to $\posq$. 
						Spoiler can reach $\posq^+$ on the output from $\posq$ on the input of the left increment gadget $\INC{}^{\text{\upshape left}}_\ell$. 
					\end{lemma}
					\begin{proof}
						Since $\INC{}^{\text{\upshape left}}_\ell$ is applicable to $\posq=(\posa,\posb,\emptyset)$ we have that $\posa(\ell)<n$, $\posa(i)=n$ for all $i>\ell$ and $\posb(i)=m$ for all $i\in[\kmo]$. Furthermore, $\posq^+=(\posa^+,\posb^+,\emptyset)$ satisfies $\posa^+(i)=\posa(i)$ for all $i<\ell$, $\posa^+(\ell)=\posa(\ell)+1$, $\posa^+(i)=1$ for all $i>\ell$ and $\posb^+(i)=1$ for all $i\in[\kmo]$. 
						By definition of the gadget all edges $\{x^i_{\posa(i)},y^i_{\posa^+(i)}\}$ are present in Spoiler's graph. Moreover, $y^i_{\posb^+(i)}=y^i_{1}$ is the only neighbor of $x^i_{\posb(i)}=x^i_{m}$ in the corresponding block. Hence, Spoiler can reach $\posq^+$ on the output by pebbling along these edges in the same way as described in Lemma~\ref{lem:inc_right_Sp}.
					\end{proof}
					Again, Duplicator does not lose when Spoiler increments a position. Furthermore, if the increment gadget is not applicable to the current configuration, then Spoiler does not reach any valid position on the output. The next lemma ensures this and is the analogue of Lemma~\ref{lem:inc_right_Du} for left increment gadgets.
					\begin{lemma}\label{lem:inc_left_Du}
						Let $\ell\in[\kmo]$ and $\posq=(\posa,\posb,T)$ be a configuration.
						\begin{itemize}
							\item[1.] If $\INC{}^{\text{\upshape left}}_\ell$ is applicable to $\posq$, then there is a winning strategy for Duplicator with boundary function $h^x_{\posq}$ on the input and $h^y_{\posq^+}$ on the output.
							\item[2.] If $\INC{}^{\text{\upshape left}}_\ell$ is not applicable to $\posq$, then there is a winning strategy for Duplicator with boundary function $h^x_{\posq}$ on the input and $h^y_{\posq_{\text{\upshape inv}}}$ on the output for an invalid configuration 
							$\posq_{\text{\upshape inv}}$.
						\end{itemize}
					\end{lemma}
					\begin{proof}
						Assume that $\INC{}^{\text{\upshape left}}_\ell$ is applicable to $\posq$.
						As in the proof of Lemma~\ref{lem:inc_right_Du} it is straightforward to check that $h=h^x_{\posq}\cup h^y_{\posq^+}$ is a homomorphism from Spoiler's side of the gadget to Duplicator's side: Every edge $\{x^i_j,x^i_{j'}\}$ is mapped to either $\{x^i_0,y^i_0\}$ or $\{x^i_{\posb(i)},y^i_{\posb^+(i)}\}=\{x^i_{m},y^i_{1}\}$. Hence, $\wp(h)$ is a winning strategy for Duplicator with the desired boundary function. 
						To prove the second statement assume that $\INC{}^{\text{\upshape left}}_\ell$ is not applicable to $\posq$ and hence $T^{\text{\upshape left}}_\ell(\posq)\neq\emptyset$. Duplicator plays according to $h^x_{\posq}$ on the input vertices. 
						If Spoiler pebbles some vertex $y^i_j$ for an $i\in T^{\text{\upshape left}}_\ell(\posq)$ we consider two cases. 
						The first case is that $\posb(i)\neq m$ or $i\in T$. Then $y^i_0$ is the neighbor of every $h^x_{\posq}(x^i_j)$ and Duplicator can safely move to $y^i_0$ (as in the proof of Lemma~\ref{lem:inc_right_Du}). The second case is that ($i=\ell$ and $\posa(\ell)=n$) or ($i>\ell$ and $\posa(i)\neq n$). 
						In this situation Duplicator can answer with $y^i_0$ since on the one hand if $h^x_{\posq}(x^i_j)=x^i_0$, 
						then there is an edge $\{x^i_0,y^i_0\}$ on Duplicator's side. 
						On the other hand, if $h^x_{\posq}(x^i_j)\neq x^i_0$, then by definition $x^i_j=x^i_{\posa(i)}$. Since there is no edge between $x^i_{\posa(i)}$ and the $i$th block of the output vertices in Spoiler's graph, the choice of $y^i_0$ extends to a partial homomorphism. 
						If Spoiler pebbles $y^i_j$ for some $i\notin T^{\text{\upshape left}}_\ell(\posq)$, then Duplicator answers with $y^i_1$ as in part 1 of this lemma. Formally, Duplicator's strategy is $\wp(h^x_{\posq}\cup h^y_{\posq_{\text{\upshape inv}}})$ where where $\posq_{\text{\upshape inv}}=(\posa_{\text{\upshape inv}},\posb_{\text{\upshape inv}},T^{\text{\upshape right}}_\ell(\posq))$ with $\posb_{\text{\upshape inv}}(i)=1$ for all $i\in[\kmo]$ and  
						\begin{align*}		
						\posa_{\text{\upshape inv}}(i) &=  \begin{cases}
						\text{arbitrary, if }i\in T^{\text{\upshape right}}_\ell(\posq), \\
						\text{$j$, such that $y^i_j$ is the neighbor of $x^i_{\posa(i)}$, otherwise.}\mbox{\qed}
						\end{cases}  
						\end{align*}
					\end{proof}
}

The \emph{switch} is an extension of the ``multiple input one-way switch'' defined in \cite{Ber13} (which in turn is a generalization of \cite{Kolaitis.2003}).
The difference is that the old switch can only be used for the case $n=1$. 
It requires some work to adjust the old switch to make it work for the more general setting. 
\confORarxiv{
	But since these modifications require a deeper inspection into this technical construct (and are not the main contribution of this paper), we refer to the full version of the paper and use the switch as black box at this point.
	We briefly explain the strategies on the switch and provide them in Lemma~\ref{lem:multi}.
}{
	The technical details of the switch are shifted to Appendix~\ref{sec:Aswitch}. At this point we focus on a high level description of the strategies and formalize them in Lemma~\ref{lem:multi}. The proof of this lemma can also be found in Appendix~\ref{sec:Aswitch}.

}
As mentioned earlier, Spoiler can simply move a valid position from the input to the output of the switch (Lemma \ref{lem:multi}(i)). 
Duplicator has a winning strategy called \textit{output strategy}, where any position is on the output and $h^{x}_{\boldsymbol{0}}$ is on the input (Lemma \ref{lem:multi}(ii)).
This ensures that Spoiler cannot move backwards to reach $\posq$ on the input from $\posq$ on the output. Hence, this strategy forces Spoiler to play through the switches in the intended direction (as indicated by arrows Figure~\ref{fig:CSP-Prop-kCons-overview}). 
Furthermore, for every invalid $\posq_{\text{\upshape inv}}$ Duplicator has a winning strategy where $h^x_{\posq_{\text{\upshape inv}}}$ is on the input and $h^y_{\boldsymbol{0}}$ is on the output (Lemma \ref{lem:multi}(iii)), which ensures that Spoiler cannot move invalid positions through the switch. 
This strategy is used by Duplicator whenever Spoiler plays on an increment gadget that is not applicable. By Lemma~\ref{lem:inc}, Duplicator can force an invalid configuration on the output of that increment gadget and hence on the input of the subsequent switch.

To ensure that Spoiler picks up all pebbles when reaching $\posq$ on the output from $\posq$ on the input, Duplicator has a critical \textit{input strategy} with $\posq$ on the input and $h^y_{\boldsymbol{0}}$ on the output  (Lemma \ref{lem:multi}(iv)). The critical positions are either contained in an output strategy, where $\posq$ is on the output, or (for technical reasons) in a \emph{restart strategy}. 
If Duplicator plays according to this input strategy, the only way for Spoiler to bring $\posq$ from the input to the output is to pebble an output critical position inside the switch (using all the pebbles) and force Duplicator to switch to the corresponding output strategy.

\spaceconstraints{
					\begin{figure}[p]
					\rotatebox{90}{%
					\begin{minipage}{\textheight}
					 \centering
					  \input{switch.tex}
					 \caption{Subgraph of the switch. On Spoiler's side, all inner-block edges are present and the inter-block edges are indicated. For the first block on Duplicator's side, all inner-block edges are drawn. Note that there is no edge between $a^i_{s,l}$ and $b^i_{0,l}$.}\label{fig:multi}
					 \end{minipage}
					}
					\end{figure}
}

	\begin{lemma} \label{lem:multi}
					 For every configuration $\posq=(\posa,\posb,T)$, the following statements hold in the existential $(\kmo+1)$-pebble game on the switch:
					\begin{enumerate}
					 \item[(i)] If $\posq$ is valid, then Spoiler can reach $\posq$ on the output from $\posq$ on the input.
					 \item[(ii)] Duplicator has a winning strategy $\mathcal H^\text{out}_{\posq}$ with boundary function $h^x_{\boldsymbol{0}}\cup h^y_{\posq}$.
					 \item[(iii)] If $\posq$ is invalid, then Duplicator has a winning strategy $\mathcal H^\text{restart}_{\posq}$ with boundary function $h^x_{\posq}\cup h^y_{\boldsymbol{0}}$.
					 \item[(iv)] If $\posq$ is valid, then Duplicator has a critical strategy $\mathcal H^\text{in}_{\posq}$ with boundary function $h^x_{\posq}\cup h^y_{\boldsymbol{0}}$ and sets of restart critical positions $\mathcal C^\text{restart-crit}_{\posq,t}$ (for $t\in [\kmo]$) and output critical positions $\mathcal C^\text{out-crit}_{\posq}$ such that:
					  \begin{enumerate}
					    \item $\crit(\mathcal H^\text{in}_{\posq}) = \bigcup_{t\in[\kmo]} \mathcal C^\text{restart-crit}_{\posq,t}\cup \mathcal C^\text{out-crit}_{\posq}$,
					    \item $\mathcal C^\text{restart-crit}_{\posq,t} \subseteq \mathcal H^\text{restart}_{(\posa,\posb,\{t\})}$ and
					    \item $\mathcal C^\text{out-crit}_{\posq} \subseteq \mathcal H^\text{out}_{\posq}$.
					  \end{enumerate}
					\end{enumerate}
					\end{lemma}

\spaceconstraints{
					In order to define the \textit{switch} we construct the two graphs: $M_S$ for Spoiler's side and $M_D$ for Duplicator's side. 
					Let
					\begin{align*}
					  V(M_S) = &\{x^i_j,a^i_j,b^i_j,y^i_j\mid i\in [\kmo],j\in[n]\}, \\
					  E(M_S) = &\big\{\{x^i_j,a^i_j\},\{a^i_j,b^i_j\},\{b^i_j,y^i_j\}\mid i\in[\kmo],j\in[n]\big\} \\
					  \cup &\big\{\{a^{i}_j,a^{i'}_{j'}\},\{b^{i}_j,b^{i'}_{j'}\},\{a^{i}_j,b^{i'}_{j'}\}\mid i,i'\in[\kmo]; i\neq i'; j,j'\in[n]\big\} 
					\end{align*}
					That is, within one block $i\in[\kmo]$ of $M_S$ the vertices $a^i_1,a^i_2,\ldots$ are pairwise connected to $b^i_1,b^i_2,\ldots$ and between two blocks $i$ and $i'$ every vertex $a^i_j$ and $b^i_j$ from block $i$ is connected to every vertex $a^{i'}_{j'}$ and $b^{i'}_{j'}$ from block $i'$. 
					For Duplicator's side of the graph, we define for $i\in [\kmo]$: 
					\begin{align*}
					 X^i &= \{x^i_s\mid 0\leq s\leq m\},& Y^i &= \{y^i_s\mid 0\leq s\leq m\}\\
					 A^i_+ &= \{a^i_{s,l}\mid s\in [m],l\in [\kmo]\}, 
					 &A^i &= A^i_+ \cup \{a^i_0\} \\
					 B^i_+ &= \{b^i_{s,l} \mid s \in [m],l\in [\kmo]\},
					 &B^i &= B^i_+ \cup \{b^i_{0,l}\mid l\in [\kmo]\}. 
					\end{align*}
					 The set of vertices of $M_D$ is 
					$$
					V(M_D) = \bigcup_{i\in [\kmo]}\left(X^i\cup A^i \cup B^i \cup Y^i\right).
					$$
					The graphs consist of $\kmo$ blocks, where the $i$-th block contains all vertices with upper index $i$. Furthermore there are four types of variables (drawn in one row in Figure~\ref{fig:CSP-Prop-kCons-overview}) the input vertices $x$, the output vertices $y$, the vertices $a$ and $b$ (with several indices). Every block of every type of vertices gets a unique color. That is, all $x^i_j$ ($y^i_j,a^i_j,b^i_j$) in $M_S$ get the same color as the vertices $X^i$ ($Y^i,A^i,B^i$, resp.) in $M_D$. This ensures that Duplicator always has to answer with vertices of the same type in the same block. 

					Now we describe the edges in $M_D$. We first define the inner-block edges $E^i$, which are also shown in Figure \ref{fig:CSP-Prop-kCons-overview}, and then the inter-block edges $E^{i,j}$:
					\begin{align*}
					 E^i =  
					  &\big(\{x^i_0\}\times A^i\big) 
					  &&\text{(E1)}
					  \\
					  &\cup 
					  \big\{ \{x^i_s, a^i_{s,l}\}\mid s\in[m]; l\in[\kmo]\big\} 
					  &&\text{(E2)}
					  \\
					  &\cup
					  \big(\{a_0^i\}\times B^i\big) 
					  &&\text{(E3)}
					  \\
					  &\cup 
					  \big\{ \{a^i_{s,l}, b^i_{s,l}\}\mid s\in[m]; l\in[\kmo]\big\} 
					  &&\text{(E4)}
					  \\
					  &\cup
					  \big\{\{a^i_{s,l},b^i_{0,l'}\}\mid s\in[m]; l,l'\in[\kmo]; l\neq l'\big\} 
					  &&\text{(E5)}
					  \\
					  &\cup
					  \big\{ \{b^i_{s,l},y^i_s\}\mid s\in[m]; l\in[\kmo]\big\} 
					  &&\text{(E6)}
					  \\
					  &\cup 
					  \big\{\{b^i_{0,l},y^i_s\}\mid s\in[m]\cup\{0\}; l\in[\kmo]\big\}, 
					  &&\text{(E7)}
					  \\
					E^{i,j} = 
					  &\big\{\{a^i_{s,l},a^j_{s',l'}\},\mid s,s'\in[m]; l,l'\in[\kmo]; l\neq l' \big\} 
					  &&\text{(E8)}
					  \\
					  &\cup
					  \big\{\{b^i_{s,l},b^j_{s',l'}\}\mid s\in[m],s'\in[m]\!\cup\!\{0\}; l,l'\in[\kmo]; l\neq l' \big\} 
					  &&\text{(E9)}
					  \\
					  &\cup
					  \big\{\{b^i_{0,l},b^j_{0,l'}\}\mid l,l'\in[\kmo] \big\} 
					  &&\text{(E10)}
					  \\
					  &\cup
					  \big\{\{a^i_{s,l},b^j_{s',l'}\}\mid s\in[m]; s'\!\in\![m]\!\cup\!\{0\}; l,l'\in[\kmo]; l\neq l' \big\} 
					  &&\text{(E11)}
					  \\  
					  &\cup
					  \big\{\{a^i_{0},a^j_{s,l}\}\mid s\in[m]; l\in[\kmo] \big\} 
					  &&\text{(E12)}
					  \\
					  &\cup
					  \big\{\{a^i_{0},b^j_{s,l}\}\mid s\in[m]\cup\{0\}; l\in[\kmo]\big\}
					  &&\text{(E13)}
					\end{align*}
					Finally, $E(M_D) = \bigcup_{i\in[\kmo]}E^i\cup \bigcup_{i,j\in[\kmo];i\neq j}E^{i,j}$.
					The next lemma states the main properties of the switch. For this, recall the definition of critical strategies (Definition~\ref{def:criticalStrategy} on page~\pageref{def:criticalStrategy}). 
					The first statement (i) states that Spoiler can pebble a valid position from the input to the output. 
					Duplicator uses the critical input strategies (iv) to ensure that Spoiler has to pebble a critical position inside the switch while he pebbles the valid position through the switch. Duplicator's output strategy (ii) ensures that Spoiler cannot move backwards (\ie, reach $\posq$ on the input from $\posq$ on the output). The restart strategy (iii) makes sure that Spoiler cannot pebble an invalid position through the switch.
					\begin{lemma} \label{lem:multi}
					 For every configuration $\posq=(\posa,\posb,T)$, the following statements hold in the existential $(\kmo+1)$-pebble game on the switch:
					\begin{enumerate}
					 \item[(i)] If $\posq$ is valid, then Spoiler can reach $\posq$ on the output from $\posq$ on the input.
					 \item[(ii)] Duplicator has a winning strategy $\mathcal H^\text{out}_{\posq}$ with boundary function $h^x_{\boldsymbol{0}}\cup h^y_{\posq}$.
					 \item[(iii)] If $\posq$ is invalid, then Duplicator has a winning strategy $\mathcal H^\text{restart}_{\posq}$ with boundary function $h^x_{\posq}\cup h^y_{\boldsymbol{0}}$.
					 \item[(iv)] If $\posq$ is valid, then Duplicator has a critical strategy $\mathcal H^\text{in}_{\posq}$ with boundary function $h^x_{\posq}\cup h^y_{\boldsymbol{0}}$ and sets of restart critical positions $\mathcal C^\text{restart-crit}_{\posq,t}$ (for $t\in [\kmo]$) and output critical positions $\mathcal C^\text{out-crit}_{\posq}$ such that:
					  \begin{enumerate}
					    \item $\crit(\mathcal H^\text{in}_{\posq}) = \bigcup_{t\in[\kmo]} \mathcal C^\text{restart-crit}_{\posq,t}\cup \mathcal C^\text{out-crit}_{\posq}$,
					    \item $\mathcal C^\text{restart-crit}_{\posq,t} \subseteq \mathcal H^\text{restart}_{(\posa,\posb,\{t\})}$ and
					    \item $\mathcal C^\text{out-crit}_{\posq} \subseteq \mathcal H^\text{out}_{\posq}$.
					  \end{enumerate}
					\end{enumerate}
					\end{lemma}

					\begin{IEEEproof}
					Let $\posq=(\posa,\posb,T)$ be an arbitrary configuration. 
					We first construct the strategy for Spoiler to prove (i). 
					Starting from position $\{(x^{1}_{\posa(1)},x^1_{\posb(1)}),\ldots,(x^{\kmo}_{\posa(\kmo)},x^{\kmo}_{\posb(\kmo)})\}$, Spoiler places the ($\kmo+1$)st pebble on $a^1_{\posa(1)}$. 
					Duplicator has to answer with $a^1_{\posb(1),l_1}$ for some $l_1\in[\kmo]$, mapping the edge $\{x^1_{\posa(1)},a^1_{\posa(1)}\}$ to some edge in (E2). 
					Next, Spoiler picks up the pebble from $x^1_{\posa(1)}$ and puts it on $a^2_{\posa(2)}$. Again, Duplicator has to answer with $a^2_{\posb(2),l_2}$ for some $l_2\in [\kmo]\setminus\{l_1\}$. 
					The index $l_2$ has to be different from $l_1$ because there is an edge between $a^1_{\posa(1)}$ and $a^2_{\posa(2)}$, but none between $a^1_{\posb(1),l_1}$ and $a^2_{\posb(2),l_1}$ in (E8). 
					Following that scheme, Spoiler can reach the position $\{(a^{1}_{\posa(1)},a^1_{\posb(1),l_1}),\ldots,(a^{\kmo}_{\posa(\kmo)},a^{\kmo}_{\posb(\kmo),l_{\kmo}})\}$ for pairwise distinct $l_1, l_2, \cdots, l_{\kmo}$. 
					Now, Spoiler pebbles $b^1_{\posa(1)}$ with the free pebble and Duplicator has to answer with a vertex in $B^1$ (due to the vertex-colors) that is adjacent to all $a^1_{\posb(1),l_1},\ldots,a^\kmo_{\posb(\kmo),l_\kmo}$. 
					This is only the case for $b^1_{\posb(1),l_1}$ (due to (E4) and (E11)), since every vertex of the form $b^1_{0,l_i}$ is not adjacent to the vertex $a^i_{\posb(i),l_i}$ according to (E5) and (E11). Furthermore, $b^1_{\posb(1),l_1}$ is the only vertex of the form $b^1_{s,l}$ (for $s>0$) that is adjacent to $a^i_{\posb(i),l_i}$.
					In the next step Spoiler picks up the pebble from $a^1_{\posa(1)}$ and puts it on $b^2_{\posa(2)}$. 
					Duplicator has to answer with a vertex that is adjacent to all vertices $b^1_{\posb(1),l_1}, a^2_{\posb(2),l_2},\ldots,a^\kmo_{\posb(\kmo),l_\kmo}$. Because of the missing edges in (E5), (E11) and (E9) (!) the only vertex with this property is $b^2_{\posb(2),l_2}$. Again, Spoiler picks up the pebble from $a^2_{\posa(2)}$ and puts it on $b^3_{\posa(3)}$. By the same argument as before, Duplicator has to answer with $b^3_{\posb(3),l_3}$, which is the only vertex adjacent to all of $b^1_{\posb(1),l_1}, b^2_{\posb(2),l_2}, a^3_{\posb(3),l_3},\ldots,a^\kmo_{\posb(\kmo),l_\kmo}$.
					Thus, Spoiler can reach $\{(b^{1}_{\posa(1)},b^1_{\posb(1),l_1}),\ldots,(b^{\kmo}_{\posa(\kmo)},b^{\kmo}_{\posb(\kmo),l_{\kmo}})\}$ and from there he reaches $\{(y^{1}_{\posa(1)},y^1_{\posb(1)}),\ldots,(y^{\kmo}_{\posa(\kmo)},y^{\kmo}_{\posb(\kmo)})\}$ by successively pebbling the edges $\{b^{i}_{\posa(i)},y^{i}_{\posa(i)}\}$.
					 
					In order to derive the winning strategies for Duplicator in (ii) and (iii) we consider several total homomorphisms from Spoiler's to Duplicator's side. Consider the edges (E1), (E3) and (E7) connecting \blacknode vertices with  \whitenode vertices in one block of Duplicator's side. They can be used by Duplicator to pebble a \whitenode vertex when Spoiler moves upwards. This is the crucial ingredient for Duplicator's output strategies (ii). 
					The first homomorphism is used when Spoiler plays the above strategy to get a valid position through the switch and has already taken all his pebbles from the input vertices. If he tries to pebble input vertices again, then Duplicator can move to $x^i_0$ and plays according to the following homomorphism:
					\begin{align*}
					h^\text{out}_{\posq,\sigma}(x^i_j) &= x^i_0 \\ 
					h^\text{out}_{\posq,\sigma}(a^i_{\posa(i)}) &= a^i_{\posb(i),\sigma(i)} 
					&
					h^\text{out}_{\posq,\sigma}(a^i_{j}) &= a^i_{0}\text{, }j\neq\posa(i)  \\
					h^\text{out}_{\posq,\sigma}(b^i_{\posa(i)}) &= b^i_{\posb(i),\sigma(i)} 
					&
					h^\text{out}_{\posq,\sigma}(b^i_{j}) &= b^i_{0,\sigma(j)}\text{, }j\neq\posa(i)  \\
					h^\text{out}_{\posq,\sigma}(y^i_{\posa(i)}) &= y^i_{\posb(i)}
					&
					h^\text{out}_{\posq,\sigma}(y^i_j) &= y^i_0\text{, }j\neq\posa(i) 
					\end{align*}
					where $\sigma \in S_{\kmo}$ is some permutation on $[\kmo]$. 
					The next homomorphism is used by Duplicator when there is some valid or invalid configuration $\posq$ at the output of the switch.
					\begin{align*}
					h^\text{out}_{\posq}(x^i_j) &= x^i_0 \\ 
					h^\text{out}_{\posq}(a^i_{j}) &= a^i_{0} \\
					h^\text{out}_{\posq}(b^i_{j}) &= b^i_{0,j} \\
					h^\text{out}_{\posq}(y^i_j) &= h^y_{\posq}(y^i_j) 
					\end{align*}
					Since $h^{\text{out}}_{\posq}$ and all $h^\text{out}_{\posq,\sigma}$ are total, 
					\begin{align*}
					  \mathcal H^\text{out}_{\posq} \defi \begin{cases}
					                                         \cl(h^{\text{out}}_{\posq})\text{, } &\posq \text{ is invalid}, \\
					              \cl(h^{\text{out}}_{\posq})\cup\bigcup_{\sigma\in S_{\kmo}} \cl(h^{\text{out}}_{\posq,\sigma}), & \text{otherwise,}
					                                        \end{cases}
					\end{align*}
					is a winning strategy for Duplicator satisfying (ii).

					If a homomorphism maps all the $a^i_{\posa(i)}$ vertices to $A^i_+$, then it has to map all $b^i$ vertices to $B^i_+$. This is due to the missing edges in (E5), (E11) and has also been used in Spoiler's strategy above. On the other hand, if for at least one $i\in[\kmo]$ all $a^i_j$ are mapped to $a^i_0$, then every $b^i_j$ can be mapped to $b^i_{0,l}$, where $l$ is chosen such that $a^j_{\posb(j),l}$ is not in the image of the homomorphism for every $j$. Duplicator benefits from this, because she can now map the $y^i_j$ vertices arbitrarily using the edges (E7). This behavior is used in the following restart strategies. Note that a homomorphism mapping some $a^i_j$ to $a^i_0$ also maps $x^i_j$ to $x^i_0$, hence restart strategies require invalid input positions.
					For invalid $\posq=(\posa,\posb,T)$, let $\mathcal H^{\text{restart}}_{\posq} \defi \{\cl(h)\mid h\in H^{\text{restart}}_{\posq}\}$, where $H^{\text{restart}}_{\posq}$ is the set of total homomorphisms $h$ satisfying the constraints $h(x^i_j) = h^x_{\posq}(x^i_j)$ and $h(y^i_j) = y^i_0$. 
					This set clearly satisfies (iii). 
					As an example fix some $t\in T$ and let $g \in H^{\text{restart}}_{\posq}$ be the following homomorphism:
					\begin{align*}
					g(x^i_j) &= h^x_{\posq}(x^i_{j}), \\ 
					g(a^i_{j}) &= a^i_{\posb(i),i}\text{, if }j=\posa(i)\text{ and }i\notin T\text{, } 
					g(a^i_{j}) = a^i_{0}\text{, otherwise,}  \\
					g(b^i_{j}) &= b^i_{0,t}, \\
					g(y^i_{j}) &= y^i_{0}.
					\end{align*}
					It remains to consider the critical input strategies (iv). They formalize the following behavior of Duplicator at the time when Spoiler wants to pebble a configuration $\posq$ through the switch as in (i). Fix a valid configuration $\posq=(\posa,\posb,\emptyset)$. If Spoiler pebbles $a^i_{\posa(i)}$ or $b^i_{\posa(i)}$, Duplicator answers within $A^i_+$ or $B^i\setminus B^i_+$, respectively. This allows her to answer on the boundary according to the boundary function defined in (iv). However, she may run into trouble when Spoiler places $\kmo$ pebbles on $a^i_{\posa(i)}$ and $b^i_{\posa(i)}$ vertices, because they extend to a $(\kmo+1)$-clique on Spoiler's side, but not on Duplicator's side (on the blocks $A^i_+$ and $B^i\setminus B^i_+$). These positions form the critical positions where Duplicator switches to an output or restart strategy. If all $\kmo$ pebbles are on $a^1_{\posa(1)},\ldots,a^{\kmo}_{\posa(\kmo)}$, as in Spoiler's strategy (i), then Duplicator switches to the output strategy (\ie, she plays according to a homomorphism $h^{\text{out}}_{\posq,\sigma}$). In all other cases she switches to a restart strategy.
					For all $\ell\in[\kmo]$ and permutations $\sigma$ on $[\kmo]$ we define partial homomorphism $h^{\text{in}}_{\posq,\sigma,\ell}$ as follows:

					\begin{align*}
					h^{\text{in}}_{\posq,\sigma,\ell}(x^i_j) &= h^x_{\posq}(x^i_{j})    \\
					h^{\text{in}}_{\posq,\sigma,\ell}(a^i_{\posa(i)}) &= a^i_{\posb(i),\sigma(i)}\text{, } i\neq\sigma^{-1}(\ell) \\
					h^{\text{in}}_{\posq,\sigma,\ell}(a^{i}_{\posa(i)}) &= \text{undefined, } i=\sigma^{-1}(\ell) \\
					h^{\text{in}}_{\posq,\sigma,\ell}(a^i_j) &= a^i_{0}\text{, } j\neq \posa(i) \\
					h^{\text{in}}_{\posq,\sigma,\ell}(b^i_j) &= b^i_{0,\ell} \\
					h^{\text{in}}_{\posq,\sigma,\ell}(y^i_j) &= y^i_{0}   
					\end{align*}
					We need to check that $h^{\text{in}}_{\posq,\sigma,\ell}$ defines a homomorphism from $M_S\setminus\{a^{\sigma^{-1}(\ell)}_{\posa(i)}\}$ to $M_D$. For most parts this is easy to verify. The important part is to check that we do not map edges to the missing pairs in the edge sets (E5), (E8) and (E11) where we require that the indices $l$ and $l'$ have to be different. The constraints of (E8) are fulfilled because of the permutation $\sigma$. The constraints of (E5) and (E11) are satisfied because we have chosen $\ell$ such that no vertex maps to $a^i_{s,\ell}$ for all $i\in[\kmo]$ and $s\in[m]$. This also shows that the partial homomorphism cannot be extended to a total homomorphism (where $h^{\text{in}}_{\posq,\sigma,\ell}$ is defined on $a^{i}_{\posa(i)}$ for $i=\sigma^{-1}(\ell)$).
					Now we define a partial homomorphism $h^{\text{in}}_{\posq,\sigma}$ for every permutation $\sigma \in S_{\kmo}$.
					\begin{align*}
					h^{\text{in}}_{\posq,\sigma}(x^i_j) &= h^x_{\posq}(x^i_{j}), \\
					h^{\text{in}}_{\posq,\sigma}(a^i_{\posa(i)}) &= a^i_{\posb(i),\sigma(i)}, \\
					h^{\text{in}}_{\posq,\sigma}(a^i_j) &=  a^i_{0}\text{, } j\neq \posa(i), \\
					 h^{\text{in}}_{\posq,\sigma}(b^i_j) &= \text{undefined},\\
					h^{\text{in}}_{\posq,\sigma}(y^i_j) &= y^i_{0}.
					\end{align*}
					Again it is not hard to see that $h^{\text{in}}_{\posq,\sigma}$ defines a partial homomorphism from $M_S$ to $M_D$. We cannot extend this partial homomorphism to a total homomorphism, because if we map $b^i_{\posa(i)}$ to some $b^i_{0,l}$ we will map to a missing edge in (E5) or (E11). Otherwise, if we chose some $b^i_{\posb(i),l}$, we will map the edge $\{b^i_{\posa(i)},y^i_{\posa(i)}\}$ in $M_S$ to the non-edge $\{b^i_{\posb(i),l},y^i_0\}$ in $M_D$. Duplicator's input strategy is the family of all subsets of all mappings $h^{\text{in}}_{\posq,\sigma,\ell}$ and $h^{\text{in}}_{\posq,\sigma}$.
					We are ready to define the critical positions.
					For all $\sigma\in S_{\kmo}$ let 
					$$
					h^{\text{out-crit}}_{\posq,\sigma} \defi  \{(a^i_{\posa(i)},a^i_{\posb(i),\sigma(i)})\mid i\in [\kmo]\}
					$$
					and for all $\sigma\in S_{\kmo}$ and $t,u\in [\kmo]$ and $s\in[n]$
					$$
					 h^{\text{restart-crit}}_{\posq,\sigma,t,u,s} \defi \{(a^i_{\posa(i)}, a^i_{\posb(i),\sigma(i)})\mid i\in[\kmo]\setminus \{t\}\}\cup\{(b^u_{s},b^u_{0,\sigma(t)})\}.
					$$
					Now we can define the sets used in (iv):
					\begin{align*}
					 \mathcal H^\text{in}_{\posq} &= \{\cl(h^{\text{in}}_{\posq,\sigma})\mid \sigma\in S_{\kmo}\}\cup\{\cl(h^{\text{in}}_{\posq,\sigma,\ell})\mid \sigma\in S_{\kmo}, \ell\in [\kmo]\}, \\
					\mathcal C^\text{out-crit}_{\posq} &= \{h^{\text{out-crit}}_{\posq,\sigma}\mid \sigma\in S_{\kmo}\}, \\
					\mathcal C^\text{restart-crit}_{\posq,t} &= \{h^{\text{restart-crit}}_{\posq,\sigma,t,u,s}\mid \sigma\in S_{\kmo}, u\in [\kmo],s\in[n]\},\\
					\crit(\mathcal H^\text{in}_{\posq}) &= \bigcup_{t\in[\kmo]} \mathcal C^\text{restart-crit}_{\posq,t}\cup \mathcal C^\text{out-crit}_{\posq}.
					\end{align*}
					First note that $h^{\text{out-crit}}_{\posq,\sigma} \subset h^{\text{in}}_{\posq,\sigma}$ and $h^{\text{restart-crit}}_{\posq,\sigma,t,u,s}\subset h^{\text{in}}_{\posq,\sigma,\sigma(t)}$. It holds that $\crit(\mathcal H^\text{in}_{\posq})\subseteq \mathcal H^\text{in}_{\posq}$.
					It easily follows from the definitions, that $h^{\text{out-crit}}_{\posq,\sigma}\subset h^{\text{out}}_{\posq,\sigma}$. Furthermore, every $h^{\text{restart-crit}}_{\posq,\sigma,t,u,s}$ can be extended to a homomorphism $g\in \mathcal H^\text{restart}_{(\posa,\posb,\{t\})}$ by defining 

					\begin{align*}
						g(x^i_j) &= h^x_{(\posa,\posb,\{t\})}(x^i_j), \\
						g(a^i_{\posa(i)}) &= h^{\text{restart-crit}}_{\posq,\sigma,t,u,s}(a^i_{\posa(i)}) = a^i_{\posb(i),\sigma(i)}\text{, if }i\neq t,\\
						g(a^t_{\posa(t)}) &= a^t_0,\\
						g(a^i_{j}) &= a^i_0\text{, if }j\neq \posa(i),\\
						g(b^i_{j}) &= b^i_{\sigma(t)},\\
						g(y^i_j) &= y^i_0.\\
					\end{align*}
					This proves statement b) and c) from (iv). It remains to show that $\mathcal H^\text{in}_{\posq}$ is a critical strategy with critical positions $\crit(\mathcal H^\text{in}_{\posq})$.
					\begin{claim}
					 For all $g\in \mathcal H^\text{in}_{\posq}$ with $|g|\leq \kmo$, either $g\in \crit(\mathcal H^\text{in}_{\posq})$ or for all $z\in V(M_S)$ there exist an $h\in \mathcal H^\text{in}_{\posq}$, such that $g\subseteq h$ and $z\in\dom(h)$.
					\end{claim}
					\begin{innerproof} As $g$ is a partial homomorphism from $\mathcal H^\text{in}_{\posq}$ (which only contains subsets of $h^{\text{in}}_{\posq,\sigma,\ell}$ and $h^{\text{in}}_{\posq,\sigma}$), we can fix some $\sigma\in S_{\kmo}$ and $\ell\in [\kmo]$ such that $g$ is a subset of the following mapping
					 \begin{align*}
					 x^i_{\posa(i)} &\mapsto x^i_{\posb(i)}, & x^i_{j} &\mapsto x^i_{0}\text{, if }j\neq\posa(i), \\
					 a^i_{\posa(i)} &\mapsto a^i_{\posb(i),\sigma(i)}, & a^i_{j} &\mapsto a^i_{0}\text{, if }j\neq\posa(i), \\
					 b^i_j &\mapsto b^i_{0,\ell}, \\
					 y^i_j &\mapsto y^i_{0}.
					 \end{align*}
					Let %
					$B_S\defi \{b^i_j\mid i\in[\kmo],j\in[n]\}\subseteq V(M_S)$.

					\noindent
					\textbf{Case 1: $|\dom(g)\cap \{a^i_{\posa(i)}\mid i\in[\kmo]\}|=\kmo$.} 
					In this case, $g=h^{\text{out-crit}}_{\posq,\sigma}$ and hence, $g\in\crit(\mathcal H^\text{in}_{\posq})$.

					\noindent
					\textbf{Case 2: $|\dom(g)\cap \{a^i_{\posa(i)}\mid i\in[\kmo]\}|=\kmo-1$.} 
					If $\dom(g)\cap B_S \neq \emptyset$, then $g=h^{\text{restart-crit}}_{\posq,\sigma,\sigma^{-1}(l),u,s}$ for some $u\in[\kmo]$ and $s\in[n]$. Thus, we can assume that $\dom(g)\cap B_S = \emptyset$ and show for all $z$ that $g$ satisfies the extension property. If $z=a^i_j$, then $h^{\text{in}}_{\posq,\sigma}$ extends $g$. If $z=x^i_j$,$z=b^i_j$ or $z=y^i_j$, then $h^{\text{in}}_{\posq,\sigma,\ell}$ extends $g$.

					\noindent
					\textbf{Case 3: $|\dom(g)\cap \{a^i_{\posa(i)}\mid i\in[\kmo]\}|\leq \kmo-2$.}
					Let $j_1$ and $j_2$ be two distinct indices such that $a^{j_1}_{\posa(j_1)}$, $a^{j_2}_{\posa(j_2)}\notin \dom(g)$. Furthermore, we can without loss of generality assume that $\sigma(j_1) = \ell$. For $z\neq a^{j_1}_{\posa(j_1)}$ the homomorphism $h^{\text{in}}_{\posq,\sigma,\ell}$ extends $g$. If $z=a^{j_1}_{\posa(j_1)}$, then $h^{\text{in}}_{\posq,\sigma',\ell}$ extends $g$, where 
					$\sigma' \defi \{(i,\sigma(i))\mid i\in [\kmo]\setminus\{j_1,j_2\}\}\cup \{(j_1,\sigma(j_2)),(j_2,\sigma(j_1))\}.$
					\end{innerproof}
					\end{IEEEproof}
}

\spaceconstraints{
				\begin{figure}[htp]
				 \centering
				  \input{init.tex}
				 \caption{The initialization gadget (for $\kmo=3$, $m=4$, $n=5$, $\posa(1)=3$, $\posa(2)=1$, $\posb(3)=5$, $\posb(1)=2$, $\posb(2)=4$, $\posb(3)=3$.)} \label{fig:init}
				\end{figure}
}

At the beginning of the game we want that Spoiler can reach the start configuration $\alpha^{-1}(0)$ on $\xnode$, which is the pebble position $\{(\xnode_1^1,\xnode_1^1),\ldots,(\xnode^k_1,\xnode^k_1)\}$. 
To ensure this, we use the \emph{initialization gadget} and identify its output vertices $y^i_j$ with the block of $\xnode^i_j$ vertices. 
\confORarxiv{
	As for the switch, this gadget is an extension of the initialization gadget presented in \cite{Ber13} and we use it as a black box here. 
	The strategies on this gadget are provided in Lemma~\ref{lem:init},
	the proof of Lemma~\ref{lem:init} is given in the full version of the paper.
}{
	As for the switch, this gadget is an extension of the initialization gadget presented in \cite{Ber13}. We now describe the strategies on the initialization gadget and formalize them in Lemma~\ref{lem:init}.
	A detailed description of the gadget and a proof of Lemma~\ref{lem:init} is given in Appendix~\ref{sec:Ainit}.
}
The main property of the gadget is that Spoiler can reach the start position $\posq$ at the boundary (i) and Duplicator has a corresponding counter strategy (ii) in this situation. 
Furthermore, if an arbitrary position occurs at the boundary during the game, Duplicator has a strategy to survive (iii). This is only a critical strategy, but Duplicator can switch to the initial strategy (hence ``restart'' the game) if Spoiler moves to one of the critical positions.

\spaceconstraints{
					The initialization gadget {\upshape INIT$^{\posq}$} is built out of two switches $M^{1}$ and $M^{2}$, vertices $z$ in Spoiler's graph and $z_1$, $z_2$ in Duplicator's graph. The three vertices $z,z_1,z_2$ share one unique vertex color. 
					Additionally, there are output boundary vertices $y^i_j$ of the usual form. The vertices $z,z_1,z_2$ and the boundary vertices are connected to $M^1$ and $M^2$ as shown in Figure \ref{fig:init} for a specific valid configuration $\posq=(\posa,\posb,\emptyset)$. Lemma \ref{lem:init} (i)--(iii) provides the strategies on {\upshape INIT$^{\posq}$}. 
}

\begin{lemma}\label{lem:init}
Let $\posq=\alpha^{-1}(0)$. The following holds in the existential $(\kmo+1)$-pebble game on {\upshape INIT}:
\begin{enumerate}
 \item[(i)] Spoiler can reach $\posq$ on the output.
 \item[(ii)] There is a winning strategy $\mathcal I^{\text{init}}$ for Duplicator with boundary function $h^y_{\posq}$.
 \item[(iii)] For every (valid or invalid) configuration $\posq'$ there is a critical strategy $\mathcal I^{\text{init}}_{\posq'}$ with boundary function $h^y_{\posq'}$ and $\crit(\mathcal I^{\text{init}}_{\posq'})\subseteq \mathcal I^{\text{init}}$.
\end{enumerate}
\end{lemma}

\spaceconstraints{
					Spoiler's strategy is quite simple. First he pebbles $z$. Duplicator has to answer with either $z_1$ or $z_2$. Then Spoiler can reach $\{(x^{i}_{\posa(i)},x^{i}_{\posb(i)})\mid i\in [\kmo]\}$ by pebbling through either $M^1$ or $M^2$. To construct the strategies for Duplicator, we can combine the strategies of the switches $M^1$ and $M^2$ such that she plays an input strategy on one switch and a restart or output strategy on the other switch. Assume that Spoiler reaches a critical position on the switch where Duplicator plays the input strategy, say $M^1$. Duplicator can now flip the strategies such that she plays a restart or output strategy on $M^1$, depending on which kind of critical position Spoiler has reached, and an input strategy on $M^2$. 

					\begin{IEEEproof}[Proof of Lemma \ref{lem:init}]
					 We start with developing the strategy for Spoiler (i). First, Spoiler pebbles $z$. Duplicator has to response with either $z_{1}$ or $z_{2}$. Depending on Duplicator's choice, Spoiler can reach either $\{(a^{i}_{\posa(i)},a^{i}_{\posb(i)})\mid i\in[\kmo]\}$ or $\{(b^{i}_{\posa(i)},b^{i}_{\posb(i)})\mid i\in[\kmo]\}$. By Lemma \ref{lem:multi}.(i) Spoiler reaches $\{(c^{i}_{\posa(i)},c^{i}_{\posb(i)})\mid i\in[\kmo]\}$ ($\{(d^{i}_{\posa(i)},d^{i}_{\posb(i)})\mid i\in[\kmo]\}$) and from there he can reach the position $\{(y^{i}_{\posa(i)},y^{i}_{\posb(i)})\mid i\in[\kmo]\}$.
					For Duplicator's strategies we start with a discussion of possible moves outside of the switches. At the top of the gadget Duplicator can map $z$ to $z_1$ and is then forced to answer with $h^a_{\posq}$ at the input of $M^1$ and for some $R\subseteq[\kmo]$ with $h^b_{(\posa,\posb,R)}$ at the input of $M^2$. On the other hand, Duplicator can map $z$ to $z_2$ and play according to $h^a_{(\posa,\posb,R)}$ and $h^b_{\posq}$.
					At the bottom of the switch the following three combinations define partial homomorphisms for all configurations $\posq'$:
					\begin{align*}  
					h^c_{\boldsymbol 0}\cup h^d_{\boldsymbol 0}\cup h^y_{\posq'} \\
					h^c_{\posq}\cup h^d_{\boldsymbol 0}\cup h^y_{\posq} \\
					h^c_{\boldsymbol 0}\cup h^d_{\posq}\cup h^y_{\posq}
					\end{align*}
					Now we can combine these partial strategies with the strategies on the switches described in Lemma \ref{lem:multi}. In strategy $\mathcal I^{\text{in-$i$}}_{t,\posq'}$ Duplicator plays an input strategy on switch $i$, a restart strategy on the other switch and according to an arbitrary configuration $\posq'$ on the $y$-block. These strategies were combined to the critical strategy $\mathcal I^{\text{init}}_{\posq'}$ described in (iii).\footnote{In the definition of the combined strategies we use the operator $\circ$ as defined in the paragraph before Lemma~\ref{lem:combcirc}. The careful reader might notice that we do not connect ``connectable strategies'' on gadgets and thus cannot apply Lemma~\ref{lem:combcirc} literally. However, by defining the edges outside of the switches as additional gadget one could arrange the definitions (with an ugly overload of notation) to fit.} 
					\begin{align*}
					\mathcal I^{\text{in-1}}_{t,\posq'} &\defi  
					\cl(\{(z,z_{1})\}) \circ 
					\mathcal H^{\text{in}}_{\posq}\langle M^{1}\rangle \circ 
					\mathcal H^{\text{restart}}_{(\posa,\posb,\{t\})}\langle M^{2}\rangle \circ 
					\cl(h^y_{\posq'}) \\
					\mathcal I^{\text{in-2}}_{t,\posq'} &\defi  
					\cl(\{(z,z_{2})\}) \circ  
					\mathcal H^{\text{restart}}_{(\posa,\posb,\{t\})}\langle M^{1}\rangle \circ 
					\mathcal H^{\text{in}}_{\posq}\langle M^{2}\rangle \circ 
					\cl(h^y_{\posq'}) \\
					\mathcal I^{\text{init}}_{\posq'} &\defi 
					\bigcup_{t\in [\kmo]} (\mathcal I^{\text{in-1}}_{t,\posq'}\cup \mathcal I^{\text{in-2}}_{t,\posq'})
					\end{align*}
					All critical positions of $\mathcal I^{\text{in-$i$}}_{t,\posq'}$ are restart or output critical positions on the switch $M^i$. By Lemma \ref{lem:multi}.(iv).(b) every restart critical position of $\mathcal I^{\text{in-1}}_{t,\posq'}$ is contained in one of the strategies $\mathcal I^{\text{in-2}}_{t,\posq'}$ as non-critical position. Hence, the only critical positions $\crit(\mathcal I^{\text{init}}_{\posq'})$ of the combined strategy are output critical positions on the switches. These output critical positions will be contained in the strategies $\mathcal I^{\text{init-$i$}}$ where Duplicator plays an output strategy on switch $i$. Together with $\mathcal I^{\text{init}}_{\posq}$ they form the winning strategy $\mathcal I^{\text{init}}$ from (ii).
					\begin{align*}  
					\mathcal I^{\text{init-1}} &\defi \cl(\{(z,z_{2})\}) \circ \mathcal H^{\text{out}}_{\posq}\langle M^{1}\rangle \circ \mathcal H^{\text{in}}_{\posq}\langle M^{2}\rangle \circ \cl(h^y_{\posq}) \\
					\mathcal I^{\text{init-2}} &\defi \cl(\{(z,z_{1})\}) \circ \mathcal H^{\text{in}}_{\posq}\langle M^{1}\rangle \circ \mathcal H^{\text{out}}_{\posq}\langle M^{2}\rangle \circ \cl(h^y_{\posq}) \\
					\mathcal I^{\text{init}} &\defi \mathcal I^{\text{init-1}} \cup \mathcal I^{\text{init-2}} \cup  \mathcal I^{\text{init}}_{\posq}
					 \end{align*}
					 $\mathcal I^{\text{init}}$ is a union of critical strategies with boundary function $h^y_{\posq}$.
					To prove that $\mathcal I^{\text{init}}$ is indeed a winning strategy on the gadget, we apply Lemma \ref{lem:combunion} and show that every critical position of one strategy is contained as non-critical position in another strategy. 
					Critical positions are inside the input strategy $\mathcal H^{\text{in}}_{\posq}$ on one of the switches. 
					By Lemma \ref{lem:multi}.(iv) they are either contained in an output or restart strategy on the corresponding switch. Hence, all restart critical positions on $M^1$ and $M^2$ are contained in $\mathcal I^{\text{init}}_{\posq}$ and all output critical positions on $M^1$ ($M^2$) are contained in $\mathcal I^{\text{init-1}}$ ($\mathcal I^{\text{init-2}}$). 
					Recall the notation $\mathcal{\widehat{S}} \defi \mathcal S\setminus \crit(\mathcal S)$, by Lemma \ref{lem:multi}.(iv) we get:
					\begin{align*}
					\crit(\mathcal I^{\text{in-2}}_{R,\posq'}) &= \crit(\mathcal I^{\text{init-1}}) 
					= \crit(\mathcal H^{\text{in}}_{\posq}\langle M^{2}\rangle)\\
					&\subseteq \mathcal H^\text{out}_{\posq}\langle M^{2}\rangle \cup \bigcup_{t\in[\kmo]} \mathcal H^\text{restart}_{(\posq,\{t\})}\langle M^{2}\rangle \\
					&\subseteq \mathcal {\widehat I}^{\text{init-2}} \cup \bigcup_{t\in[\kmo]} \mathcal{\widehat I}^{\text{in-1}}_{\{t\},\posq}, \\
					\crit(\mathcal I^{\text{in-1}}_{R,\posq'}) &= \crit(\mathcal I^{\text{init-2}}) 
					= \crit(\mathcal H^{\text{in}}_{\posq}\langle M^{1}\rangle) \\
					&\subseteq \mathcal H^\text{out}_{\posq}\langle M^{1}\rangle \cup \bigcup_{t\in[\kmo]} \mathcal H^\text{restart}_{(\posq,\{t\})}\langle M^{1}\rangle \\
					&\subseteq \mathcal {\widehat I}^{\text{init-1}} \cup \bigcup_{t\in[\kmo]} \mathcal{\widehat I}^{\text{in-2}}_{\{t\},\posq}. 
					\end{align*}
					Hence, $\crit(\mathcal I^{\text{init}}_{\posq'})\subseteq \mathcal I^{\text{init}}$ and $\mathcal I^{\text{init}}$ is a winning strategy by Lemma \ref{lem:combunion}.
					\end{IEEEproof}
}

\subsection{Proof of Theorem~\ref{thm:kConsDepthLowerBound}}

The size of the vertex set in every gadget is linear in $n$ on Spoiler's side and linear in $m$ on Duplicator's side. 
Since the overall construction uses a constant number of gadgets it follows that $|V(\struc{A}_n)|=O(n)$ and $|V(\struc{B}_m)|=O(m)$. 
To prove the lower bound on the number of rounds Spoiler needs to win the existential $(\kmo+1)$-pebble game we provide a sequence of critical strategies in Lemma~\ref{lem:CSP_winning_Spoiler} satisfying the properties stated in Lemma~\ref{lem:SequenceCriticalStrategies}. 
For a critical strategy $\mathcal S$ we let $\mathcal{\widehat{S}} \defi \mathcal S\setminus \crit(\mathcal S)$. 
\begin{lemma}\label{lem:CSP_winning_Spoiler}\label{lem:CSP_rounds_Spoiler}
Spoiler has a winning strategy in the existential $(\kmo+1)$-pebble game on $\struc{A}_n$ and $\struc{B}_m$. Furthermore, there is a sequence of critical strategies for Duplicator $\mathcal G^{\text{\upshape start}},\mathcal F_1,\mathcal G_1,\mathcal F_2, \mathcal G_2,\ldots,\mathcal G_{n^{\kmo}m^{\kmo}-2},\mathcal F_{n^{\kmo}m^{\kmo}-1}$ such that 
\begin{align*}
	\crit(\mathcal G^{\text{\upshape start}}) &\subseteq \widehat{\mathcal F}_1, \\
	\crit(\mathcal G_{i}) &\subseteq \widehat{\mathcal F}_{i+1} \cup \widehat{\mathcal G}^{\text{\upshape start}},& &1\leq i \leq n^{\kmo}m^{\kmo}-2,\\
	\crit(\mathcal F_{i}) &\subseteq \widehat{\mathcal G}_{i} \cup \widehat{\mathcal G}^{\text{\upshape start}},& &1\leq i \leq n^{\kmo}m^{\kmo}-2.
\end{align*}
\end{lemma} 
\begin{proof}[Proof of Theorem~\ref{thm:kConsDepthLowerBound}]
	For $k=2$ the theorem follows from \cite{BerVer13}. For $k\geq 3$ consider the structures $\struc{A}_{n}$ and $\struc{B}_{m}$ (for $\kmo=k-1$) defined above. 
	By Lemma~\ref{lem:CSP_winning_Spoiler} Spoiler wins the existential $k$-pebble game on $\struc{A}_{n}$ and $\struc{B}_{m}$. Furthermore, it follows via Lemma~\ref{lem:SequenceCriticalStrategies} that Spoiler needs at least $\Omega(n^{k-1}m^{k-1})$ rounds to win the game. To get  structures with exactly $n$ and $m$ vertices we take the largest $n',m'$ such that $|V(\struc{A}_{n'})|\leq n$, $|V(\struc{B}_{m'})|\leq m$ and fill up the structures with an appropriate number of isolated vertices.
\qed\end{proof}

\begin{proof}[Proof of Lemma~\ref{lem:CSP_winning_Spoiler}]
	To show that Spoiler has a winning strategy it suffices to prove the following three statements:
	\begin{enumerate}
    \item[(1)] Spoiler can reach the position $\alpha^{-1}(0)$ on $\xnode$ from $\emptyset$,
	\item[(2)] Spoiler can reach $\alpha^{-1}(i+1)$ on $\xnode$ from $\alpha^{-1}(i)$ on $\xnode$ (for $i<n^{\kmo}m^{\kmo}-1$) and
	\item[(3)] Spoiler wins from $\alpha^{-1}(n^{\kmo}m^{\kmo}-1)$ on $\xnode$.
	\end{enumerate}
	Assertion (1) follows from Lemma~\ref{lem:init} and (3) is ensured by the winning gadget. For (2), Spoiler starts with the position $\posq=\alpha^{-1}(i)$ on $\xnode$. Since $i<n^{\kmo}m^{\kmo}-1$ there is exactly one increment gadget applicable to $\posq$. Spoiler uses Lemma~\ref{lem:inc} to reach $\posq^+=\alpha^{-1}(i+1)$ on the output of that gadget. By applying Lemma~\ref{lem:multi}.(i) twice, Spoiler can pebble $\posq^+$ through the two switches to the $\xnode$ vertices.

	To define the sequence of global critical strategies we combine the partial critical strategies on the gadgets using the $\circ$-operator. There are three types of strategies: $\mathcal G^{\text{\upshape start}}$, $\mathcal F_i$ and $\mathcal G_i$. 
	To define $\mathcal G_i$ we let $\posq=\alpha^{-1}(i)$. Duplicator plays according to $h^{\xnode}_{\posq}$ on $\xnode$ and according to $h^{\ynode}_{\boldsymbol 0}$ on $\ynode$. She plays according to this strategy in the case when Spoiler reaches ``$\posq$ on $\xnode$''. The critical strategy $\mathcal G_i$ is the combination of the following (pairwise connectable) strategies on the gadgets:
	\begin{itemize}
		\item The critical strategy $\mathcal I^{\text{\upshape init}}_{\posq}$ on the initialization gadget (Lemma~\ref{lem:init}).
		\item The winning strategy with boundary $h^x_{\posq}$ and $h^y_{\posq^+}$ on the increment gadget applicable to $\posq$ (Lemma~\ref{lem:inc}).
		\item The critical input strategy $\mathcal H^{\text{in}}_{\posq^+}$ on the switch following the applicable increment gadget (Lemma~\ref{lem:multi}).
		\item The winning strategy with boundary $h^x_{\posq}$ and $h^y_{\posq_{\text{inv}}}$ on the other increment gadgets not applicable to $\posq$ (Lemma~\ref{lem:inc}). 
		\item The winning strategy $\mathcal H^{\text{restart}}_{\posq_{\text{inv}}}$ on the switches following the inapplicable increment gadgets (Lemma~\ref{lem:multi}). Here, $\posq_{\text{inv}}$ is the invalid configuration on the output of the corresponding increment gadget. 
		\item The output winning strategy $\mathcal H^{\text{out}}_{\posq}$ on the single switch (Lemma~\ref{lem:multi}).
	\end{itemize}	
	If in the above setting Spoiler increments $\posq$ through the applicable increment gadget and moves $\posq^+=\alpha^{-1}(i+1)$ through the subsequent switch, then Duplicator switches to the strategy $\mathcal F_{i+1}$. 
	To define $\mathcal F_{i}$ we fix $\posq=\alpha^{-1}(i)$. In this strategy, Duplicator plays according to $h^{\xnode}_{\boldsymbol 0}$ on $\xnode$ and according to $h^{\ynode}_{\posq}$ on $\ynode$. This critical strategy is the combination of the following strategies on the gadgets.
	\begin{itemize}
		\item The critical strategy $\mathcal I^{\text{\upshape init}}_{\boldsymbol 0}$ on the initialization gadget.
		\item The winning strategy with boundary $h^x_{\boldsymbol 0}$ and $h^y_{\boldsymbol 0}$ on the increment gadgets.
		\item The output strategy $\mathcal H^{\text{out}}_{\posq}$ on the switches following the  increment gadgets.
		\item The critical input strategy $\mathcal H^{\text{in}}_{\posq}$ on the single switch.
	\end{itemize}	
	The critical positions in the strategies $\mathcal G_i$ and $\mathcal F_i$ are inside the switches and the initialization gadget. 
	Recall that by Lemma \ref{lem:multi}.(iv) the critical positions on the switch can be divided into restart critical positions and output critical positions. 
	Furthermore, all output critical positions of $\mathcal G_i$, which are inside the switch following the applicable increment gadget, are contained as non-critical positions in $\mathcal F_{i+1}$.
	All output critical position in $\mathcal F_{i}$, which are inside the single switch, are contained as non-critical positions in $\mathcal G_{i}$.
	Now we define $\mathcal G^{\text{start}}$, which contains all other critical positions of $\mathcal G_{i}$ and $\mathcal F_{i}$. 
	The critical strategy $\mathcal G^{\text{start}}$ is the union of several other global strategies. The first one is $\mathcal G^{\text{\upshape init}}$, which is defined as $\mathcal G_0$ except that it contains the winning strategy $\mathcal I^{\text{\upshape init}}$ on the initialization gadget.
	Thus, by Lemma~\ref{lem:init}, it contains every critical position on the initialization gadget as non-critical position. 
	Note that the output critical positions of $\mathcal G^{\text{\upshape init}}$ are contained as non-critical positions in $\mathcal F_1$. 
	Since $\mathcal G^{\text{\upshape init}}$ handles the critical positions on the initialization gadget and we discussed the output critical positions on the switches, it remains to consider the restart critical positions of the strategies. 
	For this we construct a strategy $\mathcal G^{\text{restart}}_i$ to handle the restart critical positions of $\mathcal G_i$ (for $i\geq 1$) and of $\mathcal G^{\text{init}}$ (for $i=0$). Furthermore, we define for every $i\geq 1$ a strategy $\mathcal F^{\text{restart}}_i$ to handle the restart critical positions of $\mathcal F_i$.
		
	For $0\leq i\leq n^{\kmo}m^{\kmo}-2$ and $t\in[\kmo]$ we let $\posq=\alpha^{-1}(i)=(\posa,\posb,\emptyset)$ and $\posq_{t}$ be the invalid configuration $(\posa,\posb,\{t\})$. The global strategy $\mathcal G^{\text{restart}}_{i,t}$ is the combination of the following strategies on the gadgets.
	\begin{itemize}
		\item The critical strategy $\mathcal I^{\text{\upshape init}}_{\posq_t}$ on the initialization gadget.%
		\item The winning strategy with boundary $h^x_{\posq_t}$ and $h^y_{\posq_{\text{inv}}}$ on the  increment gadgets. Note that, since $\posq_t$ is invalid, no increment gadget is applicable to $\posq_t$. 
		\item The winning strategy $\mathcal H^{\text{restart}}_{\posq_{\text{inv}}}$ on the switches following the increment gadgets. Again, $\posq_{\text{inv}}$ is the invalid configuration at the output of the preceding increment gadget.
		\item The output winning strategy $\mathcal H^{\text{out}}_{\posq_t}$ on the single switch.
	\end{itemize}
	Finally, we let $\mathcal G^{\text{restart}}_{i}\defi \bigcup_{i\in[\kmo]}\mathcal G^{\text{restart}}_{i,t}$. 
	Note that by Lemma \ref{lem:multi}.(iv) every restart critical position of $\mathcal G_i$ is contained in $\mathcal G^{\text{restart}}_{i}$ and every restart critical position of $\mathcal G^{\text{\upshape init}}$ is contained in $\mathcal G^{\text{restart}}_{0}$. 
	Now we define for $1\leq i\leq n^{\kmo}m^{\kmo}-2$, $t\in[\kmo]$, $\posq=\alpha^{-1}(i)=(\posa,\posb,\emptyset)$ and $\posq_{t}\defi (\posa,\posb,\{t\})$ the strategy $\mathcal F^{\text{restart}}_{i,t}$ analogously. It consists of the following partial strategies.
	\begin{itemize}
		\item The critical strategy $\mathcal I^{\text{\upshape init}}_{\boldsymbol 0}$ on the initialization gadget.%
		\item The winning strategy with boundary $h^x_{\boldsymbol 0}$ and $h^y_{\boldsymbol 0}$ on the  increment gadgets. 
		\item The winning strategy $\mathcal H^{\text{restart}}_{\boldsymbol 0}$ on the switches after the increment gadgets. 
		\item The winning strategy $\mathcal H^{\text{restart}}_{\posq_t}$ on the single switch.
	\end{itemize}
	In the end we let $\mathcal F^{\text{restart}}_{i}$ be the union of all $\mathcal F^{\text{restart}}_{i,t}$. Note that every restart critical position of $\mathcal F_{i}$ is contained as non-critical position in $\mathcal F^{\text{restart}}_{i}$. Finally, let
	$$
	\mathcal G^{\text{start}} \defi \mathcal G^{\text{init}} \cup \bigcup_{0\leq i\leq n^{\kmo}m^{\kmo}-2} \mathcal G^{\text{restart}}_i \cup \bigcup_{1\leq i\leq n^{\kmo}m^{\kmo}-2}\mathcal F^{\text{restart}}_i.
	$$
	To conclude the proof note that the critical positions of $\mathcal G^{\text{restart}}_i$ and $\mathcal F^{\text{restart}}_i$ are inside the initialization gadget and hence contained in $\widehat{\mathcal G}^{\text{init}}$. Thus they are not critical positions of ${\mathcal G}^{\text{start}}$. Hence, $\crit({\mathcal G}^{\text{start}}) = \crit({\mathcal G}^{\text{init}}) \subseteq \widehat{\mathcal F}_1$. 
\qed\end{proof}

\section{Conclusion}

We have proven an optimal lower bound of $\Omega(n^{k-1}d^{k-1})$ on the number of nested propagation steps in the $k$-consistency procedure on constraint networks with $n$ variables and domain size $d$. It follows that every parallel propagation algorithm has to perform at least $\Omega(n^{k-1}d^{k-1})$ sequential steps. 
Using \mbox{$(n+d)^{O(k)}$} processors (one for every instance of the inference rule), $k$-consistency can be computed in $O(n^{k-1}d^{k-1})$ parallel time, which is optimal for propagation algorithms. In addition, the best sequential algorithm runs in $O(n^{k}d^{k})$. The overhead compared to the parallel approach is mainly caused by the time needed to search for the next inconsistent assignment that might be propagated -- and this seems to be the only task that can be parallelized. 

Although we have proven an optimal lower bound in the general setting, it might be interesting to investigate the propagation depth of $k$-consistency on restricted classes of structures. Especially, if in such cases the propagation depth is bounded by $O(\log(n+d))$, we know that $k$-consistency is in \NCclass{} and hence parallelizable.

\newpage

\bibliography{Literatur}

\newpage

\appendix

\section{Appendix}

This appendix contains proofs skipped in the main text. We first give a simple proof of the correspondence between the existential pebble game and CSP-refutations (Subsection~\ref{sec:Alemma}) followed by the definition and strategies on the switch (Subsection~\ref{sec:Aswitch}) and the initialization gadget (Subsection~\ref{sec:Ainit}).

\subsection{Proof of Lemma~\ref{lem:CSPrefutaionEXpebblegame}}\label{sec:Alemma}

\begin{lemma}\textbf{\upshape (Reminder of Lemma~\ref{lem:CSPrefutaionEXpebblegame})} \hspace*{0.5cm}
	Let $\struc{A}$ and $\struc{B}$ be two relational structures. There is a CSP-refutation for $\struc{A}$ and $\struc{B}$ of width $k-1$ and depth $d$ if and only if Spoiler has a strategy to win the existential $k$-pebble game on $\struc{A}$ and $\struc{B}$ within $d$ rounds.
\end{lemma}

\begin{proof}
	For one direction assume that there is a CSP-refutation $P$ of depth $d$ and width $k-1$. We show by induction over the depth that every partial mapping $p$ of depth $i$ occurring in the refutation defines a position of pebbles from which Spoiler can win the existential $k$-pebble game within $i$ rounds. It follows that Spoiler can win the game from $\emptyset$ (all pebbles off the board) within $d$ rounds.
	All mappings of depth $i=0$ are axioms and thus not partial homomorphisms. Hence, Spoiler wins immediately. For the induction step assume that $p$ has depth $i>0$. Therefore, $|p|<k$ and $p$ is derived from $p'_1\cup\{x\mapsto a_1\},\ldots,p'_n\cup\{x\mapsto a_n\}$ ($p'_j\subseteq p$) each of depth $<i$. Spoiler can now reach one of these positions within one round by placing the remaining pebble on $x$. Depending on Duplicator's choice (some $a_j\in V(\struc{B})$) Spoiler moves to $p_j'\cup\{x\mapsto a_j\}$ by picking up the pebbles in $p\setminus p_j'$. 
	By induction assumption, Spoiler can win from $p_j'\cup\{x\mapsto a_j\}$ within $<i$ rounds and hence he can win from $p$ within $i$ rounds. 

	To prove the other direction we show by induction over the number of rounds that if Spoiler has an $i$-round winning strategy from a position $p$, then some $p'\subseteq p$ has a CSP-derivation of depth $i$. Since we assume that Spoiler has a $d$-round winning strategy from $\emptyset$, the lemma follows. 
	For $i=0$ the 0-round winning positions are precisely the axioms in our derivation system. Assume that Spoiler has an $i$-round winning strategy from $p$. In the next round in his strategy Spoiler first has to pick up at least one pebble. Let $p'\subseteq p$ be the new position and note that $|p'|<k$. 
	By the definition of the game Spoiler also has a $i$-round winning strategy from $p'$. Let $x\in V(\struc{A})$ be the element on which the next pebble is set. Since Spoiler has a strategy to win against every possible choice of Duplicator, we know that $p'\cup\{x\mapsto a_1\},\ldots,p'\cup\{x\mapsto a_n\}$ are positions from which Spoiler can win the game within $i-1$ rounds. For all these positions there is a $p_j\subseteq p'\cup\{x\mapsto a_j\}$ that has a derivation of depth at most $i-1$ by induction assumption. If for some $j$ it holds that $p_j\subseteq p'\subseteq p$ we are done. Otherwise, all $p_j$ are of the form $p_j=p_j'\cup\{x\mapsto a_j\}$ with $p_j'\subseteq p'\subseteq p$. Thus, $p$ has a derivation of depth at most $i$ by applying the derivation rule \eqref{eq:CSPrule1}.
\end{proof}

\subsection{The Switch}\label{sec:Aswitch}

\begin{figure}[p]
\rotatebox{90}{%
\begin{minipage}{\textheight}
 \centering
  \input{switch.tex}
 \caption{Subgraph of the switch. On Spoiler's side, all inner-block edges are present and the inter-block edges are indicated. For the first block on Duplicator's side, all inner-block edges are drawn. Note that there is no edge between $a^i_{s,l}$ and $b^i_{0,l}$.}\label{fig:multi}
 \end{minipage}
}
\end{figure}

For the reader familiar with the literature it is worth noting that the switch presented here is an extension of the ``multiple input one-way switch'' defined in \cite{Ber12,Ber13}. The difference is that the old switch can only be used for the case $n=1$. 
However, many strategies and technical definitions can directly be extended to this more general setting. 
The switch in \cite{Ber12,Ber13} was in turn a further development of the work from  
Kolaitis and Panttaja \cite{Kolaitis.2003}, who constructed a switch for the special case $n=1$ and $m=2$.

In order to define the \textit{switch} we construct the two graphs: $M_S$ for Spoiler's side and $M_D$ for Duplicator's side. 
Let
\begin{align*}
  V(M_S) = &\{x^i_j,a^i_j,b^i_j,y^i_j\mid i\in [\kmo],j\in[n]\}, \\
  E(M_S) = &\big\{\{x^i_j,a^i_j\},\{a^i_j,b^i_j\},\{b^i_j,y^i_j\}\mid i\in[\kmo],j\in[n]\big\} \\
  \cup &\big\{\{a^{i}_j,a^{i'}_{j'}\},\{b^{i}_j,b^{i'}_{j'}\},\{a^{i}_j,b^{i'}_{j'}\}\mid i,i'\in[\kmo]; i\neq i'; j,j'\in[n]\big\} 
\end{align*}
That is, within one block $i\in[\kmo]$ of $M_S$ the vertices $a^i_1,a^i_2,\ldots$ are pairwise connected to $b^i_1,b^i_2,\ldots$ and between two blocks $i$ and $i'$ every vertex $a^i_j$ and $b^i_j$ from block $i$ is connected to every vertex $a^{i'}_{j'}$ and $b^{i'}_{j'}$ from block $i'$. 
For Duplicator's side of the graph, we define for $i\in [\kmo]$: 
\begin{align*}
 X^i &= \{x^i_s\mid 0\leq s\leq m\},& Y^i &= \{y^i_s\mid 0\leq s\leq m\}\\
 A^i_+ &= \{a^i_{s,l}\mid s\in [m],l\in [\kmo]\}, 
 &A^i &= A^i_+ \cup \{a^i_0\} \\
 B^i_+ &= \{b^i_{s,l} \mid s \in [m],l\in [\kmo]\},
 &B^i &= B^i_+ \cup \{b^i_{0,l}\mid l\in [\kmo]\}. 
\end{align*}
 The set of vertices of $M_D$ is 
$$
V(M_D) = \bigcup_{i\in [\kmo]}\left(X^i\cup A^i \cup B^i \cup Y^i\right).
$$
The graphs consist of $\kmo$ blocks, where the $i$-th block contains all vertices with upper index $i$. Furthermore there are four types of variables (drawn in one row in Figure~\ref{fig:CSP-Prop-kCons-overview}) the input vertices $x$, the output vertices $y$, the vertices $a$ and $b$ (with several indices). Every block of every type of vertices gets a unique color. That is, all $x^i_j$ ($y^i_j,a^i_j,b^i_j$) in $M_S$ get the same color as the vertices $X^i$ ($Y^i,A^i,B^i$, resp.) in $M_D$. This ensures that Duplicator always has to answer with vertices of the same type in the same block. 

Now we describe the edges in $M_D$. We first define the inner-block edges $E^i$, which are also shown in Figure \ref{fig:CSP-Prop-kCons-overview}, and then the inter-block edges $E^{i,j}$:
\begin{align*}
 E^i =  
  &\big(\{x^i_0\}\times A^i\big) 
  &&\text{(E1)}
  \\
  &\cup 
  \big\{ \{x^i_s, a^i_{s,l}\}\mid s\in[m]; l\in[\kmo]\big\} 
  &&\text{(E2)}
  \\
  &\cup
  \big(\{a_0^i\}\times B^i\big) 
  &&\text{(E3)}
  \\
  &\cup 
  \big\{ \{a^i_{s,l}, b^i_{s,l}\}\mid s\in[m]; l\in[\kmo]\big\} 
  &&\text{(E4)}
  \\
  &\cup
  \big\{\{a^i_{s,l},b^i_{0,l'}\}\mid s\in[m]; l,l'\in[\kmo]; l\neq l'\big\} 
  &&\text{(E5)}
  \\
  &\cup
  \big\{ \{b^i_{s,l},y^i_s\}\mid s\in[m]; l\in[\kmo]\big\} 
  &&\text{(E6)}
  \\
  &\cup 
  \big\{\{b^i_{0,l},y^i_s\}\mid s\in[m]\cup\{0\}; l\in[\kmo]\big\}, 
  &&\text{(E7)}
  \\
E^{i,j} = 
  &\big\{\{a^i_{s,l},a^j_{s',l'}\},\mid s,s'\in[m]; l,l'\in[\kmo]; l\neq l' \big\} 
  &&\text{(E8)}
  \\
  &\cup
  \big\{\{b^i_{s,l},b^j_{s',l'}\}\mid s\in[m],s'\in[m]\!\cup\!\{0\}; l,l'\in[\kmo]; l\neq l' \big\} 
  &&\text{(E9)}
  \\
  &\cup
  \big\{\{b^i_{0,l},b^j_{0,l'}\}\mid l,l'\in[\kmo] \big\} 
  &&\text{(E10)}
  \\
  &\cup
  \big\{\{a^i_{s,l},b^j_{s',l'}\}\mid s\in[m]; s'\!\in\![m]\!\cup\!\{0\}; l,l'\in[\kmo]; l\neq l' \big\} 
  &&\text{(E11)}
  \\  
  &\cup
  \big\{\{a^i_{0},a^j_{s,l}\}\mid s\in[m]; l\in[\kmo] \big\} 
  &&\text{(E12)}
  \\
  &\cup
  \big\{\{a^i_{0},b^j_{s,l}\}\mid s\in[m]\cup\{0\}; l\in[\kmo]\big\}
  &&\text{(E13)}
\end{align*}
Finally, $E(M_D) = \bigcup_{i\in[\kmo]}E^i\cup \bigcup_{i,j\in[\kmo];i\neq j}E^{i,j}$.
The next lemma states the main properties of the switch. For this, recall the definition of critical strategies (Definition~\ref{def:criticalStrategy} on page~\pageref{def:criticalStrategy}). 
The first statement (i) states that Spoiler can pebble a valid position from the input to the output. 
Duplicator uses the critical input strategies (iv) to ensure that Spoiler has to pebble a critical position inside the switch while he pebbles the valid position through the switch. Duplicator's output strategy (ii) ensures that Spoiler cannot move backwards (\ie, reach $\posq$ on the input from $\posq$ on the output). The restart strategy (iii) makes sure that Spoiler cannot pebble an invalid position through the switch.
\begin{lemma} \textbf{\upshape (Reminder of Lemma~\ref{lem:multi})} \hspace*{0.5cm}
 For every configuration $\posq=(\posa,\posb,T)$, the following statements hold in the existential $(\kmo+1)$-pebble game on the switch:
\begin{enumerate}
 \item[(i)] If $\posq$ is valid, then Spoiler can reach $\posq$ on the output from $\posq$ on the input.
 \item[(ii)] Duplicator has a winning strategy $\mathcal H^\text{out}_{\posq}$ with boundary function $h^x_{\boldsymbol{0}}\cup h^y_{\posq}$.
 \item[(iii)] If $\posq$ is invalid, then Duplicator has a winning strategy $\mathcal H^\text{restart}_{\posq}$ with boundary function $h^x_{\posq}\cup h^y_{\boldsymbol{0}}$.
 \item[(iv)] If $\posq$ is valid, then Duplicator has a critical strategy $\mathcal H^\text{in}_{\posq}$ with boundary function $h^x_{\posq}\cup h^y_{\boldsymbol{0}}$ and sets of restart critical positions $\mathcal C^\text{restart-crit}_{\posq,t}$ (for $t\in [\kmo]$) and output critical positions $\mathcal C^\text{out-crit}_{\posq}$ such that:
  \begin{enumerate}
    \item $\crit(\mathcal H^\text{in}_{\posq}) = \bigcup_{t\in[\kmo]} \mathcal C^\text{restart-crit}_{\posq,t}\cup \mathcal C^\text{out-crit}_{\posq}$,
    \item $\mathcal C^\text{restart-crit}_{\posq,t} \subseteq \mathcal H^\text{restart}_{(\posa,\posb,\{t\})}$ and
    \item $\mathcal C^\text{out-crit}_{\posq} \subseteq \mathcal H^\text{out}_{\posq}$.
  \end{enumerate}
\end{enumerate}
\end{lemma}

\begin{IEEEproof}
Let $\posq=(\posa,\posb,T)$ be an arbitrary configuration. 
We first construct the strategy for Spoiler to prove (i). 
Starting from position $\{(x^{1}_{\posa(1)},x^1_{\posb(1)}),\ldots,(x^{\kmo}_{\posa(\kmo)},x^{\kmo}_{\posb(\kmo)})\}$, Spoiler places the ($\kmo+1$)st pebble on $a^1_{\posa(1)}$. 
Duplicator has to answer with $a^1_{\posb(1),l_1}$ for some $l_1\in[\kmo]$, mapping the edge $\{x^1_{\posa(1)},a^1_{\posa(1)}\}$ to some edge in (E2). 
Next, Spoiler picks up the pebble from $x^1_{\posa(1)}$ and puts it on $a^2_{\posa(2)}$. Again, Duplicator has to answer with $a^2_{\posb(2),l_2}$ for some $l_2\in [\kmo]\setminus\{l_1\}$. 
The index $l_2$ has to be different from $l_1$ because there is an edge between $a^1_{\posa(1)}$ and $a^2_{\posa(2)}$, but none between $a^1_{\posb(1),l_1}$ and $a^2_{\posb(2),l_1}$ in (E8). 
Following that scheme, Spoiler can reach the position $\{(a^{1}_{\posa(1)},a^1_{\posb(1),l_1}),\ldots,(a^{\kmo}_{\posa(\kmo)},a^{\kmo}_{\posb(\kmo),l_{\kmo}})\}$ for pairwise distinct $l_1, l_2, \cdots, l_{\kmo}$. 
Now, Spoiler pebbles $b^1_{\posa(1)}$ with the free pebble and Duplicator has to answer with a vertex in $B^1$ (due to the vertex-colors) that is adjacent to all $a^1_{\posb(1),l_1},\ldots,a^\kmo_{\posb(\kmo),l_\kmo}$. 
This is only the case for $b^1_{\posb(1),l_1}$ (due to (E4) and (E11)), since every vertex of the form $b^1_{0,l_i}$ is not adjacent to the vertex $a^i_{\posb(i),l_i}$ according to (E5) and (E11). Furthermore, $b^1_{\posb(1),l_1}$ is the only vertex of the form $b^1_{s,l}$ (for $s>0$) that is adjacent to $a^i_{\posb(i),l_i}$.
In the next step Spoiler picks up the pebble from $a^1_{\posa(1)}$ and puts it on $b^2_{\posa(2)}$. 
Duplicator has to answer with a vertex that is adjacent to all vertices $b^1_{\posb(1),l_1}, a^2_{\posb(2),l_2},\ldots,a^\kmo_{\posb(\kmo),l_\kmo}$. Because of the missing edges in (E5), (E11) and (E9) (!) the only vertex with this property is $b^2_{\posb(2),l_2}$. Again, Spoiler picks up the pebble from $a^2_{\posa(2)}$ and puts it on $b^3_{\posa(3)}$. By the same argument as before, Duplicator has to answer with $b^3_{\posb(3),l_3}$, which is the only vertex adjacent to all of $b^1_{\posb(1),l_1}, b^2_{\posb(2),l_2}, a^3_{\posb(3),l_3},\ldots,a^\kmo_{\posb(\kmo),l_\kmo}$.
Thus, Spoiler can reach $\{(b^{1}_{\posa(1)},b^1_{\posb(1),l_1}),\ldots,(b^{\kmo}_{\posa(\kmo)},b^{\kmo}_{\posb(\kmo),l_{\kmo}})\}$ and from there he reaches $\{(y^{1}_{\posa(1)},y^1_{\posb(1)}),\ldots,(y^{\kmo}_{\posa(\kmo)},y^{\kmo}_{\posb(\kmo)})\}$ by successively pebbling the edges $\{b^{i}_{\posa(i)},y^{i}_{\posa(i)}\}$.
 
In order to derive the winning strategies for Duplicator in (ii) and (iii) we consider several total homomorphisms from Spoiler's to Duplicator's side. Consider the edges (E1), (E3) and (E7) connecting \blacknode vertices with  \whitenode vertices in one block of Duplicator's side. They can be used by Duplicator to pebble a \whitenode vertex when Spoiler moves upwards. This is the crucial ingredient for Duplicator's output strategies (ii). 
The first homomorphism is used when Spoiler plays the above strategy to get a valid position through the switch and has already taken all his pebbles from the input vertices. If he tries to pebble input vertices again, then Duplicator can move to $x^i_0$ and plays according to the following homomorphism:
\begin{align*}
h^\text{out}_{\posq,\sigma}(x^i_j) &= x^i_0 \\ 
h^\text{out}_{\posq,\sigma}(a^i_{\posa(i)}) &= a^i_{\posb(i),\sigma(i)} 
&
h^\text{out}_{\posq,\sigma}(a^i_{j}) &= a^i_{0}\text{, }j\neq\posa(i)  \\
h^\text{out}_{\posq,\sigma}(b^i_{\posa(i)}) &= b^i_{\posb(i),\sigma(i)} 
&
h^\text{out}_{\posq,\sigma}(b^i_{j}) &= b^i_{0,\sigma(j)}\text{, }j\neq\posa(i)  \\
h^\text{out}_{\posq,\sigma}(y^i_{\posa(i)}) &= y^i_{\posb(i)}
&
h^\text{out}_{\posq,\sigma}(y^i_j) &= y^i_0\text{, }j\neq\posa(i) 
\end{align*}
where $\sigma \in S_{\kmo}$ is some permutation on $[\kmo]$. 
The next homomorphism is used by Duplicator when there is some valid or invalid configuration $\posq$ at the output of the switch.
\begin{align*}
h^\text{out}_{\posq}(x^i_j) &= x^i_0 \\ 
h^\text{out}_{\posq}(a^i_{j}) &= a^i_{0} \\
h^\text{out}_{\posq}(b^i_{j}) &= b^i_{0,j} \\
h^\text{out}_{\posq}(y^i_j) &= h^y_{\posq}(y^i_j) 
\end{align*}
Since $h^{\text{out}}_{\posq}$ and all $h^\text{out}_{\posq,\sigma}$ are total, 
\begin{align*}
  \mathcal H^\text{out}_{\posq} \defi \begin{cases}
                                         \cl(h^{\text{out}}_{\posq})\text{, } &\posq \text{ is invalid}, \\
              \cl(h^{\text{out}}_{\posq})\cup\bigcup_{\sigma\in S_{\kmo}} \cl(h^{\text{out}}_{\posq,\sigma}), & \text{otherwise,}
                                        \end{cases}
\end{align*}
is a winning strategy for Duplicator satisfying (ii).

If a homomorphism maps all the $a^i_{\posa(i)}$ vertices to $A^i_+$, then it has to map all $b^i$ vertices to $B^i_+$. This is due to the missing edges in (E5), (E11) and has also been used in Spoiler's strategy above. On the other hand, if for at least one $i\in[\kmo]$ all $a^i_j$ are mapped to $a^i_0$, then every $b^i_j$ can be mapped to $b^i_{0,l}$, where $l$ is chosen such that $a^j_{\posb(j),l}$ is not in the image of the homomorphism for every $j$. Duplicator benefits from this, because she can now map the $y^i_j$ vertices arbitrarily using the edges (E7). This behavior is used in the following restart strategies. Note that a homomorphism mapping some $a^i_j$ to $a^i_0$ also maps $x^i_j$ to $x^i_0$, hence restart strategies require invalid input positions.
For invalid $\posq=(\posa,\posb,T)$, let $\mathcal H^{\text{restart}}_{\posq} \defi \{\cl(h)\mid h\in H^{\text{restart}}_{\posq}\}$, where $H^{\text{restart}}_{\posq}$ is the set of total homomorphisms $h$ satisfying the constraints $h(x^i_j) = h^x_{\posq}(x^i_j)$ and $h(y^i_j) = y^i_0$. 
This set clearly satisfies (iii). 
As an example fix some $t\in T$ and let $g \in H^{\text{restart}}_{\posq}$ be the following homomorphism:
\begin{align*}
g(x^i_j) &= h^x_{\posq}(x^i_{j}), \\ 
g(a^i_{j}) &= a^i_{\posb(i),i}\text{, if }j=\posa(i)\text{ and }i\notin T\text{, } 
g(a^i_{j}) = a^i_{0}\text{, otherwise,}  \\
g(b^i_{j}) &= b^i_{0,t}, \\
g(y^i_{j}) &= y^i_{0}.
\end{align*}
It remains to consider the critical input strategies (iv). They formalize the following behavior of Duplicator at the time when Spoiler wants to pebble a configuration $\posq$ through the switch as in (i). Fix a valid configuration $\posq=(\posa,\posb,\emptyset)$. If Spoiler pebbles $a^i_{\posa(i)}$ or $b^i_{\posa(i)}$, Duplicator answers within $A^i_+$ or $B^i\setminus B^i_+$, respectively. This allows her to answer on the boundary according to the boundary function defined in (iv). However, she may run into trouble when Spoiler places $\kmo$ pebbles on $a^i_{\posa(i)}$ and $b^i_{\posa(i)}$ vertices, because they extend to a $(\kmo+1)$-clique on Spoiler's side, but not on Duplicator's side (on the blocks $A^i_+$ and $B^i\setminus B^i_+$). These positions form the critical positions where Duplicator switches to an output or restart strategy. If all $\kmo$ pebbles are on $a^1_{\posa(1)},\ldots,a^{\kmo}_{\posa(\kmo)}$, as in Spoiler's strategy (i), then Duplicator switches to the output strategy (\ie, she plays according to a homomorphism $h^{\text{out}}_{\posq,\sigma}$). In all other cases she switches to a restart strategy.
For all $\ell\in[\kmo]$ and permutations $\sigma$ on $[\kmo]$ we define partial homomorphism $h^{\text{in}}_{\posq,\sigma,\ell}$ as follows:

\begin{align*}
h^{\text{in}}_{\posq,\sigma,\ell}(x^i_j) &= h^x_{\posq}(x^i_{j})    \\
h^{\text{in}}_{\posq,\sigma,\ell}(a^i_{\posa(i)}) &= a^i_{\posb(i),\sigma(i)}\text{, } i\neq\sigma^{-1}(\ell) \\
h^{\text{in}}_{\posq,\sigma,\ell}(a^{i}_{\posa(i)}) &= \text{undefined, } i=\sigma^{-1}(\ell) \\
h^{\text{in}}_{\posq,\sigma,\ell}(a^i_j) &= a^i_{0}\text{, } j\neq \posa(i) \\
h^{\text{in}}_{\posq,\sigma,\ell}(b^i_j) &= b^i_{0,\ell} \\
h^{\text{in}}_{\posq,\sigma,\ell}(y^i_j) &= y^i_{0}   
\end{align*}
We need to check that $h^{\text{in}}_{\posq,\sigma,\ell}$ defines a homomorphism from $M_S\setminus\{a^{\sigma^{-1}(\ell)}_{\posa(i)}\}$ to $M_D$. For most parts this is easy to verify. The important part is to check that we do not map edges to the missing pairs in the edge sets (E5), (E8) and (E11) where we require that the indices $l$ and $l'$ have to be different. The constraints of (E8) are fulfilled because of the permutation $\sigma$. The constraints of (E5) and (E11) are satisfied because we have chosen $\ell$ such that no vertex maps to $a^i_{s,\ell}$ for all $i\in[\kmo]$ and $s\in[m]$. This also shows that the partial homomorphism cannot be extended to a total homomorphism (where $h^{\text{in}}_{\posq,\sigma,\ell}$ is defined on $a^{i}_{\posa(i)}$ for $i=\sigma^{-1}(\ell)$).
Now we define a partial homomorphism $h^{\text{in}}_{\posq,\sigma}$ for every permutation $\sigma \in S_{\kmo}$.
\begin{align*}
h^{\text{in}}_{\posq,\sigma}(x^i_j) &= h^x_{\posq}(x^i_{j}), \\
h^{\text{in}}_{\posq,\sigma}(a^i_{\posa(i)}) &= a^i_{\posb(i),\sigma(i)}, \\
h^{\text{in}}_{\posq,\sigma}(a^i_j) &=  a^i_{0}\text{, } j\neq \posa(i), \\
 h^{\text{in}}_{\posq,\sigma}(b^i_j) &= \text{undefined},\\
h^{\text{in}}_{\posq,\sigma}(y^i_j) &= y^i_{0}.
\end{align*}
Again it is not hard to see that $h^{\text{in}}_{\posq,\sigma}$ defines a partial homomorphism from $M_S$ to $M_D$. We cannot extend this partial homomorphism to a total homomorphism, because if we map $b^i_{\posa(i)}$ to some $b^i_{0,l}$ we will map to a missing edge in (E5) or (E11). Otherwise, if we chose some $b^i_{\posb(i),l}$, we will map the edge $\{b^i_{\posa(i)},y^i_{\posa(i)}\}$ in $M_S$ to the non-edge $\{b^i_{\posb(i),l},y^i_0\}$ in $M_D$. Duplicator's input strategy is the family of all subsets of all mappings $h^{\text{in}}_{\posq,\sigma,\ell}$ and $h^{\text{in}}_{\posq,\sigma}$.
We are ready to define the critical positions.
For all $\sigma\in S_{\kmo}$ let 
$$
h^{\text{out-crit}}_{\posq,\sigma} \defi  \{(a^i_{\posa(i)},a^i_{\posb(i),\sigma(i)})\mid i\in [\kmo]\}
$$
and for all $\sigma\in S_{\kmo}$ and $t,u\in [\kmo]$ and $s\in[n]$
$$
 h^{\text{restart-crit}}_{\posq,\sigma,t,u,s} \defi \{(a^i_{\posa(i)}, a^i_{\posb(i),\sigma(i)})\mid i\in[\kmo]\setminus \{t\}\}\cup\{(b^u_{s},b^u_{0,\sigma(t)})\}.
$$
Now we can define the sets used in (iv):
\begin{align*}
 \mathcal H^\text{in}_{\posq} &= \{\cl(h^{\text{in}}_{\posq,\sigma})\mid \sigma\in S_{\kmo}\}\cup\{\cl(h^{\text{in}}_{\posq,\sigma,\ell})\mid \sigma\in S_{\kmo}, \ell\in [\kmo]\}, \\
\mathcal C^\text{out-crit}_{\posq} &= \{h^{\text{out-crit}}_{\posq,\sigma}\mid \sigma\in S_{\kmo}\}, \\
\mathcal C^\text{restart-crit}_{\posq,t} &= \{h^{\text{restart-crit}}_{\posq,\sigma,t,u,s}\mid \sigma\in S_{\kmo}, u\in [\kmo],s\in[n]\},\\
\crit(\mathcal H^\text{in}_{\posq}) &= \bigcup_{t\in[\kmo]} \mathcal C^\text{restart-crit}_{\posq,t}\cup \mathcal C^\text{out-crit}_{\posq}.
\end{align*}
First note that $h^{\text{out-crit}}_{\posq,\sigma} \subset h^{\text{in}}_{\posq,\sigma}$ and $h^{\text{restart-crit}}_{\posq,\sigma,t,u,s}\subset h^{\text{in}}_{\posq,\sigma,\sigma(t)}$. It holds that $\crit(\mathcal H^\text{in}_{\posq})\subseteq \mathcal H^\text{in}_{\posq}$.
It easily follows from the definitions, that $h^{\text{out-crit}}_{\posq,\sigma}\subset h^{\text{out}}_{\posq,\sigma}$. Furthermore, every $h^{\text{restart-crit}}_{\posq,\sigma,t,u,s}$ can be extended to a homomorphism $g\in \mathcal H^\text{restart}_{(\posa,\posb,\{t\})}$ by defining 

\begin{align*}
	g(x^i_j) &= h^x_{(\posa,\posb,\{t\})}(x^i_j), \\
	g(a^i_{\posa(i)}) &= h^{\text{restart-crit}}_{\posq,\sigma,t,u,s}(a^i_{\posa(i)}) = a^i_{\posb(i),\sigma(i)}\text{, if }i\neq t,\\
	g(a^t_{\posa(t)}) &= a^t_0,\\
	g(a^i_{j}) &= a^i_0\text{, if }j\neq \posa(i),\\
	g(b^i_{j}) &= b^i_{\sigma(t)},\\
	g(y^i_j) &= y^i_0.\\
\end{align*}
This proves statement b) and c) from (iv). It remains to show that $\mathcal H^\text{in}_{\posq}$ is a critical strategy with critical positions $\crit(\mathcal H^\text{in}_{\posq})$.
\begin{claim}
 For all $g\in \mathcal H^\text{in}_{\posq}$ with $|g|\leq \kmo$, either $g\in \crit(\mathcal H^\text{in}_{\posq})$ or for all $z\in V(M_S)$ there exist an $h\in \mathcal H^\text{in}_{\posq}$, such that $g\subseteq h$ and $z\in\dom(h)$.
\end{claim}
\begin{innerproof} As $g$ is a partial homomorphism from $\mathcal H^\text{in}_{\posq}$ (which only contains subsets of $h^{\text{in}}_{\posq,\sigma,\ell}$ and $h^{\text{in}}_{\posq,\sigma}$), we can fix some $\sigma\in S_{\kmo}$ and $\ell\in [\kmo]$ such that $g$ is a subset of the following mapping
 \begin{align*}
 x^i_{\posa(i)} &\mapsto x^i_{\posb(i)}, & x^i_{j} &\mapsto x^i_{0}\text{, if }j\neq\posa(i), \\
 a^i_{\posa(i)} &\mapsto a^i_{\posb(i),\sigma(i)}, & a^i_{j} &\mapsto a^i_{0}\text{, if }j\neq\posa(i), \\
 b^i_j &\mapsto b^i_{0,\ell}, \\
 y^i_j &\mapsto y^i_{0}.
 \end{align*}
Let %
$B_S\defi \{b^i_j\mid i\in[\kmo],j\in[n]\}\subseteq V(M_S)$.

\noindent
\textbf{Case 1: $|\dom(g)\cap \{a^i_{\posa(i)}\mid i\in[\kmo]\}|=\kmo$.} 
In this case, $g=h^{\text{out-crit}}_{\posq,\sigma}$ and hence, $g\in\crit(\mathcal H^\text{in}_{\posq})$.

\noindent
\textbf{Case 2: $|\dom(g)\cap \{a^i_{\posa(i)}\mid i\in[\kmo]\}|=\kmo-1$.} 
If $\dom(g)\cap B_S \neq \emptyset$, then $g=h^{\text{restart-crit}}_{\posq,\sigma,\sigma^{-1}(l),u,s}$ for some $u\in[\kmo]$ and $s\in[n]$. Thus, we can assume that $\dom(g)\cap B_S = \emptyset$ and show for all $z$ that $g$ satisfies the extension property. If $z=a^i_j$, then $h^{\text{in}}_{\posq,\sigma}$ extends $g$. If $z=x^i_j$,$z=b^i_j$ or $z=y^i_j$, then $h^{\text{in}}_{\posq,\sigma,\ell}$ extends $g$.

\noindent
\textbf{Case 3: $|\dom(g)\cap \{a^i_{\posa(i)}\mid i\in[\kmo]\}|\leq \kmo-2$.}
Let $j_1$ and $j_2$ be two distinct indices such that $a^{j_1}_{\posa(j_1)}$, $a^{j_2}_{\posa(j_2)}\notin \dom(g)$. Furthermore, we can without loss of generality assume that $\sigma(j_1) = \ell$. For $z\neq a^{j_1}_{\posa(j_1)}$ the homomorphism $h^{\text{in}}_{\posq,\sigma,\ell}$ extends $g$. If $z=a^{j_1}_{\posa(j_1)}$, then $h^{\text{in}}_{\posq,\sigma',\ell}$ extends $g$, where 
$\sigma' \defi \{(i,\sigma(i))\mid i\in [\kmo]\setminus\{j_1,j_2\}\}\cup \{(j_1,\sigma(j_2)),(j_2,\sigma(j_1))\}.$
\end{innerproof}
\end{IEEEproof}

\newpage
\subsection{The Initialization Gadget}\label{sec:Ainit}

\begin{figure}[htp]
 \centering
  \input{init.tex}
 \caption{The initialization gadget (for $\kmo=3$, $m=4$, $n=5$, $\posa(1)=3$, $\posa(2)=1$, $\posb(3)=5$, $\posb(1)=2$, $\posb(2)=4$, $\posb(3)=3$.)} \label{fig:init}
\end{figure}

At the beginning of the game we want that Spoiler can reach the start configuration $\alpha^{-1}(0)$ on $\xnode$, which is the pebble position $\{(\xnode_1^1,\xnode_1^1),\ldots,(\xnode^k_1,\xnode^k_1)\}$. To ensure this, we introduce an initialization gadget and identify its output vertices $y^i_j$ with the block of $\xnode^i_j$ vertices. The main property is that Spoiler can force the start configuration on the output of the gadget. Another additional property is that from any position on the output of that gadget Duplicator does not lose. This property causes the main difficulties and is needed because other positions than the start position occur on the $\xnode$ vertices during the course of the game.
In other applications one might need to initialize the game with other configurations than $\alpha^{-1}(0)$. For this, we define the initialization gadget more generally for every valid configuration $\posq$.

The initialization gadget {\upshape INIT$^{\posq}$} is built out of two switches $M^{1}$ and $M^{2}$, vertices $z$ in Spoiler's graph and $z_1$, $z_2$ in Duplicator's graph. The three vertices $z,z_1,z_2$ share one unique vertex color. 
Additionally, there are output boundary vertices $y^i_j$ of the usual form. The vertices $z,z_1,z_2$ and the boundary vertices are connected to $M^1$ and $M^2$ as shown in Figure \ref{fig:init} for a specific valid configuration $\posq=(\posa,\posb,\emptyset)$. Lemma \ref{lem:generalInit} (i)--(iii) provides the strategies on {\upshape INIT$^{\posq}$}. The main property is that Spoiler can reach the start position $\posq$ at the boundary (i) and Duplicator has a corresponding counter strategy (ii) in this situation. 
Furthermore, if an arbitrary position occurs at the boundary during the game, Duplicator has a strategy to survive (iii).

\begin{lemma}\textbf{\upshape (A slightly more general version of Lemma~\ref{lem:init})} \hspace*{0.5cm} \label{lem:generalInit}
For every valid configuration $\posq=(\posa,\posb,\emptyset)$ the following holds in the existential $(\kmo+1)$-pebble game on {\upshape INIT$^{\posq}$}:
\begin{enumerate}
 \item[(i)] Spoiler can reach $\posq$ on the output.
 \item[(ii)] There is a winning strategy $\mathcal I^{\text{init}}$ for Duplicator with boundary function $h^y_{\posq}$.
 \item[(iii)] For every (valid or invalid) configuration $\posq'$ there is a critical strategy $\mathcal I^{\text{init}}_{\posq'}$ with boundary function $h^y_{\posq'}$ and $\crit(\mathcal I^{\text{init}}_{\posq'})\subseteq \mathcal I^{\text{init}}$.
\end{enumerate}
\end{lemma}

Spoiler's strategy is quite simple. First he pebbles $z$. Duplicator has to answer with either $z_1$ or $z_2$. Then Spoiler can reach $\{(x^{i}_{\posa(i)},x^{i}_{\posb(i)})\mid i\in [\kmo]\}$ by pebbling through either $M^1$ or $M^2$. To construct the strategies for Duplicator, we can combine the strategies of the switches $M^1$ and $M^2$ such that she plays an input strategy on one switch and a restart or output strategy on the other switch. Assume that Spoiler reaches a critical position on the switch where Duplicator plays the input strategy, say $M^1$. Duplicator can now flip the strategies such that she plays a restart or output strategy on $M^1$, depending on which kind of critical position Spoiler has reached, and an input strategy on $M^2$. 

\begin{IEEEproof}[Proof of Lemma \ref{lem:generalInit}]
 We start with developing the strategy for Spoiler (i). First, Spoiler pebbles $z$. Duplicator has to response with either $z_{1}$ or $z_{2}$. Depending on Duplicator's choice, Spoiler can reach either $\{(a^{i}_{\posa(i)},a^{i}_{\posb(i)})\mid i\in[\kmo]\}$ or $\{(b^{i}_{\posa(i)},b^{i}_{\posb(i)})\mid i\in[\kmo]\}$. By Lemma \ref{lem:multi}.(i) Spoiler reaches $\{(c^{i}_{\posa(i)},c^{i}_{\posb(i)})\mid i\in[\kmo]\}$ ($\{(d^{i}_{\posa(i)},d^{i}_{\posb(i)})\mid i\in[\kmo]\}$) and from there he can reach the position $\{(y^{i}_{\posa(i)},y^{i}_{\posb(i)})\mid i\in[\kmo]\}$.
For Duplicator's strategies we start with a discussion of possible moves outside of the switches. At the top of the gadget Duplicator can map $z$ to $z_1$ and is then forced to answer with $h^a_{\posq}$ at the input of $M^1$ and for some $R\subseteq[\kmo]$ with $h^b_{(\posa,\posb,R)}$ at the input of $M^2$. On the other hand, Duplicator can map $z$ to $z_2$ and play according to $h^a_{(\posa,\posb,R)}$ and $h^b_{\posq}$.
At the bottom of the switch the following three combinations define partial homomorphisms for all configurations $\posq'$:
\begin{align*}  
h^c_{\boldsymbol 0}\cup h^d_{\boldsymbol 0}\cup h^y_{\posq'} \\
h^c_{\posq}\cup h^d_{\boldsymbol 0}\cup h^y_{\posq} \\
h^c_{\boldsymbol 0}\cup h^d_{\posq}\cup h^y_{\posq}
\end{align*}
Now we can combine these partial strategies with the strategies on the switches described in Lemma \ref{lem:multi}. In strategy $\mathcal I^{\text{in-$i$}}_{t,\posq'}$ Duplicator plays an input strategy on switch $i$, a restart strategy on the other switch and according to an arbitrary configuration $\posq'$ on the $y$-block. These strategies were combined to the critical strategy $\mathcal I^{\text{init}}_{\posq'}$ described in (iii).
\begin{align*}
\mathcal I^{\text{in-1}}_{t,\posq'} &\defi  
\cl(\{(z,z_{1})\}) \circ 
\mathcal H^{\text{in}}_{\posq}\langle M^{1}\rangle \circ 
\mathcal H^{\text{restart}}_{(\posa,\posb,\{t\})}\langle M^{2}\rangle \circ 
\cl(h^y_{\posq'}) \\
\mathcal I^{\text{in-2}}_{t,\posq'} &\defi  
\cl(\{(z,z_{2})\}) \circ  
\mathcal H^{\text{restart}}_{(\posa,\posb,\{t\})}\langle M^{1}\rangle \circ 
\mathcal H^{\text{in}}_{\posq}\langle M^{2}\rangle \circ 
\cl(h^y_{\posq'}) \\
\mathcal I^{\text{init}}_{\posq'} &\defi 
\bigcup_{t\in [\kmo]} (\mathcal I^{\text{in-1}}_{t,\posq'}\cup \mathcal I^{\text{in-2}}_{t,\posq'})
\end{align*}
All critical positions of $\mathcal I^{\text{in-$i$}}_{t,\posq'}$ are restart or output critical positions on the switch $M^i$. By Lemma \ref{lem:multi}.(iv).(b) every restart critical position of $\mathcal I^{\text{in-1}}_{t,\posq'}$ is contained in one of the strategies $\mathcal I^{\text{in-2}}_{t,\posq'}$ as non-critical position. Hence, the only critical positions $\crit(\mathcal I^{\text{init}}_{\posq'})$ of the combined strategy are output critical positions on the switches. These output critical positions will be contained in the strategies $\mathcal I^{\text{init-$i$}}$ where Duplicator plays an output strategy on switch $i$. Together with $\mathcal I^{\text{init}}_{\posq}$ they form the winning strategy $\mathcal I^{\text{init}}$ from (ii).
\begin{align*}  
\mathcal I^{\text{init-1}} &\defi \cl(\{(z,z_{2})\}) \circ \mathcal H^{\text{out}}_{\posq}\langle M^{1}\rangle \circ \mathcal H^{\text{in}}_{\posq}\langle M^{2}\rangle \circ \cl(h^y_{\posq}) \\
\mathcal I^{\text{init-2}} &\defi \cl(\{(z,z_{1})\}) \circ \mathcal H^{\text{in}}_{\posq}\langle M^{1}\rangle \circ \mathcal H^{\text{out}}_{\posq}\langle M^{2}\rangle \circ \cl(h^y_{\posq}) \\
\mathcal I^{\text{init}} &\defi \mathcal I^{\text{init-1}} \cup \mathcal I^{\text{init-2}} \cup  \mathcal I^{\text{init}}_{\posq}
 \end{align*}
 $\mathcal I^{\text{init}}$ is a union of critical strategies with boundary function $h^y_{\posq}$.
To prove that $\mathcal I^{\text{init}}$ is indeed a winning strategy on the gadget, we show that every critical position of one strategy is contained as non-critical position in another strategy. 
Critical positions are inside the input strategy $\mathcal H^{\text{in}}_{\posq}$ on one of the switches. 
By Lemma \ref{lem:multi}.(iv) they are either contained in an output or restart strategy on the corresponding switch. Hence, all restart critical positions on $M^1$ and $M^2$ are contained in $\mathcal I^{\text{init}}_{\posq}$ and all output critical positions on $M^1$ ($M^2$) are contained in $\mathcal I^{\text{init-1}}$ ($\mathcal I^{\text{init-2}}$). 
Recall the notation $\mathcal{\widehat{S}} \defi \mathcal S\setminus \crit(\mathcal S)$, by Lemma \ref{lem:multi}.(iv) we get:
\begin{align*}
\crit(\mathcal I^{\text{in-2}}_{R,\posq'}) &= \crit(\mathcal I^{\text{init-1}}) 
= \crit(\mathcal H^{\text{in}}_{\posq}\langle M^{2}\rangle)\\
&\subseteq \mathcal H^\text{out}_{\posq}\langle M^{2}\rangle \cup \bigcup_{t\in[\kmo]} \mathcal H^\text{restart}_{(\posq,\{t\})}\langle M^{2}\rangle \\
&\subseteq \mathcal {\widehat I}^{\text{init-2}} \cup \bigcup_{t\in[\kmo]} \mathcal{\widehat I}^{\text{in-1}}_{\{t\},\posq}, \\
\crit(\mathcal I^{\text{in-1}}_{R,\posq'}) &= \crit(\mathcal I^{\text{init-2}}) 
= \crit(\mathcal H^{\text{in}}_{\posq}\langle M^{1}\rangle) \\
&\subseteq \mathcal H^\text{out}_{\posq}\langle M^{1}\rangle \cup \bigcup_{t\in[\kmo]} \mathcal H^\text{restart}_{(\posq,\{t\})}\langle M^{1}\rangle \\
&\subseteq \mathcal {\widehat I}^{\text{init-1}} \cup \bigcup_{t\in[\kmo]} \mathcal{\widehat I}^{\text{in-2}}_{\{t\},\posq}. 
\end{align*}
Hence, $\crit(\mathcal I^{\text{init}}_{\posq'})\subseteq \mathcal I^{\text{init}}$ and $\mathcal I^{\text{init}}$ is a winning strategy.
\end{IEEEproof}

\end{document}